\documentclass[11pt]{article}

\usepackage[utf8x]{inputenc}
\usepackage{scalerel}
 \usepackage{xcolor}

\usepackage[margin=1in]{geometry}

\usepackage{url}            
\usepackage{cite}
\usepackage{hyperref}

\definecolor{darkpastelgreen}{rgb}{0.01, 0.75, 0.24}
	\definecolor{cadmiumgreen}{rgb}{0.0, 0.42, 0.24}
\definecolor{armygreen}{rgb}{0.29, 0.33, 0.13}
\hypersetup{colorlinks,linkcolor={blue},citecolor={darkpastelgreen},urlcolor={red}}

\usepackage{booktabs}       
\usepackage{amsfonts}       
\usepackage{nicefrac}       
\usepackage{microtype}      
\usepackage{url}

\usepackage{amsmath}
\usepackage{color}
\usepackage{amssymb}
\usepackage{amsthm}
\usepackage{pifont}
\usepackage[ruled,vlined]{algorithm2e}

\usepackage{paralist}
\usepackage{dsfont}
\usepackage{natbib}
\usepackage{csquotes}
\usepackage{comment}
\usepackage{mathtools}

\usepackage{microtype}
\usepackage{graphicx}
\usepackage{booktabs} 
\usepackage{dirtytalk}
\usepackage{subfiles}

\usepackage{caption}
\usepackage{subcaption}

\usepackage{xspace}
\usepackage{forloop}
\usepackage{multirow}
\usepackage{algorithm,algorithmic,refcount}

\usepackage{graphicx}
\usepackage[shortlabels]{enumitem}
\usepackage{bbm}
\usepackage{float}

\title{Exploiting Shared Representations for Personalized Federated Learning}

\author{Liam Collins\thanks{Department of Electrical and Computer Engineering, 
The University of Texas at Austin, Austin, TX,  USA. \qquad\qquad\{liamc@utexas.edu, mokhtari@austin.utexas.edu, sanjay.shakkottai@utexas.edu\}.}, \quad Hamed Hassani\thanks{Department of Electrical and Systems Engineering, University of Pennsylvania, Philadelphia, PA, USA. \qquad\qquad \{hassani@seas.upenn.edu\}.} ,\quad  Aryan Mokhtari$^*$ ,\quad Sanjay Shakkottai$^*$}

\date{}

\begin{document}

\maketitle
\newcommand{\ones}{\mathbf{1}}
\newcommand{\integers}{{\mbox{\bf Z}}}
\newcommand{\symm}{{\mbox{\bf S}}}  

\newcommand{\nullspace}{{\mathcal N}}
\newcommand{\range}{{\mathcal R}}
\newcommand{\Rank}{\mathop{\bf Rank}}
\newcommand{\Tr}{\mathop{\bf Tr}}
\newcommand{\diag}{\mathop{\bf diag}}
\newcommand{\card}{\mathop{\bf card}}
\newcommand{\rank}{\mathop{\bf rank}}
\newcommand{\conv}{\mathop{\bf conv}}
\newcommand{\prox}{\mathbf{prox}}

\newcommand{\ind}{\mathds{1}}
\newcommand{\E}{\mathbb{E}}
\newcommand{\Prob}{\mathbb{P}}
\newcommand{\bigO}{\mathcal{O}}
\newcommand{\B}{\mathcal{B}}
\newcommand{\s}{\mathcal{S}}
\newcommand{\Ev}{\mathcal{E}}
\newcommand{\R}{\mathbb{R}}
\newcommand{\Co}{{\mathop {\bf Co}}} 
\newcommand{\dist}{\mathop{\bf dist{}}}
\newcommand{\argmin}{\mathop{\rm argmin}}
\newcommand{\argmax}{\mathop{\rm argmax}}
\newcommand{\epi}{\mathop{\bf epi}} 
\newcommand{\Vol}{\mathop{\bf vol}}
\newcommand{\dom}{\mathop{\bf dom}} 
\newcommand{\intr}{\mathop{\bf int}}
\newcommand{\sign}{\mathop{\bf sign}}
\def\polylog{\operatorname{polylog}}
\def\dist{\operatorname{dist}}

\newcommand{\cf}{{\it cf.}}
\newcommand{\eg}{{\it e.g.}}
\newcommand{\ie}{{\it i.e.}}
\newcommand{\etc}{{\it etc.}}

\newtheorem{innercustomthm}{Theorem}
\newenvironment{customthm}[1]
  {\renewcommand\theinnercustomthm{#1}\innercustomthm}
  {\endinnercustomthm}

\newtheorem{remark}{Remark}
\newtheorem{definition}{Definition}
\newtheorem{corollary}{Corollary}
\newtheorem{proposition}{Proposition}
\newtheorem{lemma}{Lemma}
\newtheorem{fact}{Fact}
\newtheorem{assumption}{Assumption}
\newtheorem{claim}{Claim}

\makeatletter
\newcommand*\rel@kern[1]{\kern#1\dimexpr\macc@kerna}
\newcommand*\widebar[1]{%
  \begingroup
  \def\mathaccent##1##2{%
    \rel@kern{0.8}%
    \overline{\rel@kern{-0.8}\macc@nucleus\rel@kern{0.2}}%
    \rel@kern{-0.2}%
  }%
  \macc@depth\@ne
  \let\math@bgroup\@empty \let\math@egroup\macc@set@skewchar
  \mathsurround\z@ \frozen@everymath{\mathgroup\macc@group\relax}%
  \macc@set@skewchar\relax
  \let\mathaccentV\macc@nested@a
  \macc@nested@a\relax111{#1}%
  \endgroup
}
\makeatother

\newcommand{\numberthis}{\addtocounter{equation}{1}\tag{\theequation}}

\newcommand{\simiid}{\overset{\text{i.i.d.}}{\sim}}

\begin{abstract}
Deep neural networks have shown the ability to extract universal feature representations from data such as images and text that have been useful for a variety of learning tasks. However, the fruits of representation learning have yet to be fully-realized in federated settings. Although data in federated settings is often non-i.i.d. across clients, the success of centralized deep learning suggests that data often shares a global {\em feature representation}, while the statistical heterogeneity across clients or tasks is concentrated in the {\em labels}. Based on this intuition, we propose a novel federated learning framework and algorithm for learning a shared data representation across clients and unique local heads for each client. Our algorithm harnesses the distributed computational power across clients to perform many local-updates with respect to the low-dimensional local parameters for every update of the representation. We prove that this method obtains linear convergence to the ground-truth representation with near-optimal sample complexity in a linear setting, demonstrating that it can efficiently reduce the problem dimension for each client. This result is of interest beyond federated learning to a broad class of problems in which we aim to learn a shared low-dimensional representation among data distributions, for example in meta-learning and multi-task learning.
Further,  extensive experimental results show the empirical improvement of our method over alternative personalized federated learning approaches in federated environments with  heterogeneous data. 
\end{abstract}
\newpage

\section{Introduction}
\label{sec:intro}



Many of the most heralded successes of modern machine learning have come in {\em centralized} settings, wherein a single model is trained on a large amount of centrally-stored data. The growing number of data-gathering devices, however, calls for a distributed architecture to train models. Federated learning aims at addressing this issue by providing a platform in which a group of clients collaborate to learn effective models for each client by leveraging the local computational power, memory, and data of all clients \citep{mcmahan2017communication}. The task of coordinating between the clients is fulfilled by a central server that combines the models received from the clients at each round and broadcasts the updated information to them. Importantly, the server and clients are restricted to methods that satisfy communication and privacy constraints, preventing them from directly applying centralized techniques. 

However, one of the most important challenges in federated learning is the issue of \textit{data heterogeneity}, where the underlying data distribution of client tasks could be substantially different from each other. In such settings, if the server and clients learn a single shared model (e.g., by minimizing average loss), the resulting model could perform poorly for many of the clients in the network (and also not generalize well across diverse data \citep{jiang2019improving}). 
In fact, for some clients, it might be better to simply use their own local data (even if it is small) to train a local model; see Figure~\ref{fig:fedrep1}. Finally, the (federated) trained model may not generalize well to unseen clients that have not participated in the training process. These issues raise the question: 
\begin{quote}
\vspace{-0mm}
  ``\textit{How can we exploit the data and computational power of all clients in data heterogeneous settings to learn a personalized model for each client?}''
  \vspace{-0mm}
\end{quote}
We address this question by taking  advantage of the common 
representation among clients. Specifically, we view the data heterogeneous federated learning problem as $n$ parallel learning tasks that they possibly have some common structure, and {\em our goal is to learn and exploit this common representation to improve the quality of each client's model}. This approach draws inspiration from centralized learning, where we have  witnessed success in training multiple tasks or learning multiple classes simultaneously by leveraging a common  (low-dimensional) representation (e.g. in image classification, next-word prediction) \citep{bengio2013representation, lecun2015deep}. 

\noindent \textbf{Main Contributions.} We introduce a novel federated learning framework and an associated algorithm for data heterogeneous settings. We summarize our main contributions  below.


\begin{figure}[t]
\begin{center}
\centerline{\includegraphics[width=0.55\columnwidth]{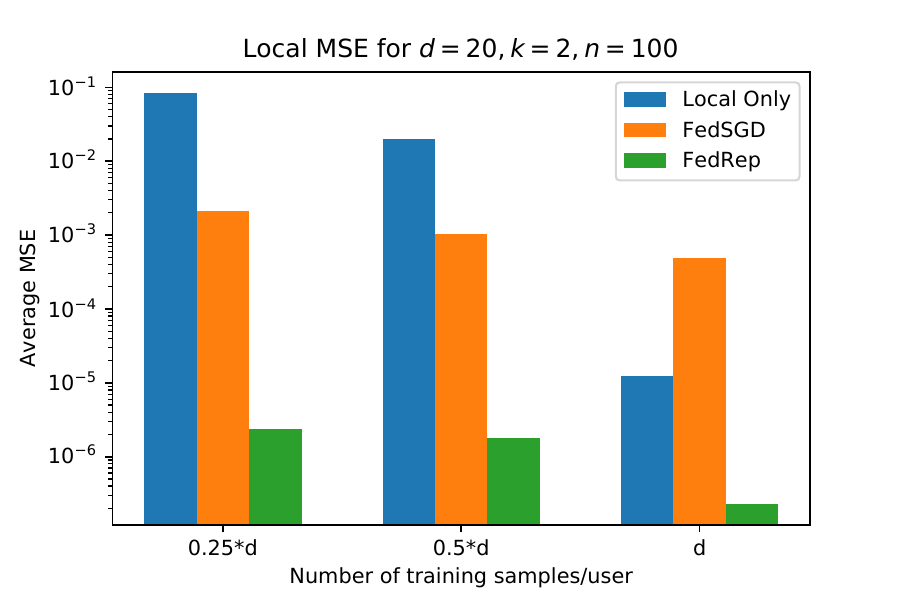}}
\vspace{-2mm}
\caption{Local only training suffers in small-training data regimes, whereas training a single global model with FedSGD cannot overcome client heterogeneity even when the number of training samples is large. FedRep exploits a common representation of the clients to achieve small error in all cases.}
\label{fig:fedrep1}
\end{center}
\vskip -0.3in
\end{figure}

\begin{itemize}
    \item[(i)]  \textbf{FedRep Algorithm.} {Federated Representation Learning (FedRep)} leverages all of the data stored across clients to learn a global low-dimensional representation using gradient-based updates. Further, it enables each client to compute a personalized, low-dimensional classifier, which we term as the client's head, that accounts for the unique labeling of  each client's local data.
    
    
    
    
    
\item[(ii)] \textbf{Optimization for linear representation learning.} 
    We show that FedRep converges to the ground-truth representation at an {\em exponentially fast rate} in the case that each client aims to solve a linear regression problem with a two-layer linear neural network. In this special case, we reduce FedRep to alternating minimization (for the heads)-descent (for the representation). 
    Our analysis shows  that this simple algorithm requires only $\mathcal{O}( ( \nicefrac{d}{n}+ \log  (n))\log(\nicefrac{1}{\epsilon}))$ samples per client to reach an $\epsilon$-accurate representation, where $n$ is the number of clients and $d$ is the dimension of the data. This result is of interest beyond federated learning since it shows that alternating minimization-descent efficiently solves  the  linear multi-task representation learning problem considered in \cite{maurer2016benefit,tripuraneni2020provable,du2020fewshot}.

    \item[(iii)] \textbf{Empirical Results.} Through a combination of synthetic and real datasets (CIFAR10, CIFAR100, FEMNIST, Sent140) we show the benefits of FedRep in: (a) leveraging many local updates, (b) robustness to different levels of heterogeneity, and (c) generalization to new clients. 
    Our experiments indicate that FedRep outperforms several important baselines in heterogeneous settings that share a global representation.



\end{itemize}

\noindent \textbf{Benefits of FedRep.} FedRep has numerous advantages over standard federated learning  (in which a single model is learned):

{\em (I) Provable gains of cooperation.} From our sample complexity bounds, it follows that with FedRep, the sample complexity per client scales as $\Theta(  \nicefrac{d}{n}+ \log(n))$. On the other hand, local learning (without any collaboration) has a sample complexity that scales as $\Theta(d).$ Thus, if $1 \ll n \ll e^{\Theta(d)}$ (see Section~\ref{sec:fed-analysis} for details), we expect benefits of collaboration through federation. When $d$ is large (as is typical in practice), $e^{\Theta(d)}$ is exponentially larger, and federation helps each client. { \em To the best of our knowledge, this is the first sample-complexity-based result for personalized federated learning that demonstrates the benefit of cooperation.}

{\em (II) Generalization to new clients.} For a new client, since a ready-made representation is available, the client only needs to learn a \textcolor{black}{head} with a low-dimensional representation of dimension $k$. Thus, its sample complexity scales only as $\Theta(k)$  instead of $\Theta(d)$ if no representation is learned.

{\em (III) More local updates.} By reducing the problem dimension, each client can make many local updates at each communication round, which is beneficial in learning its own individual head. This is unlike standard federated learning where multiple local updates in a heterogeneous setting moves each client {\em away} from the best averaged representation, and thus {\em hurts} performance.

\subsection{Related Work.}  \label{sec:rw}
\textbf{Personalized Federated Learning.}
A variety of recent works have studied personalization in federated learning using, for example, local fine-tuning \citep{wang2019federated, yu2020salvaging}, meta-learning \citep{ chen2018federated,khodak2019adaptive, jiang2019improving,fallah2020personalized}, additive mixtures of local and global models \citep{hanzely2020federated, deng2020adaptive, mansour2020three}, and multi-task learning \citep{smith2017federated}. In all of these methods, each client's subproblem is still full-dimensional - there is no notion of learning a dimensionality-reduced set of local parameters. 
More recently, \citet{liang2020think} also proposed a representation learning method for federated learning, but their method attempts to learn many local representations and a single global head as opposed to a single global representation and many local heads.
Earlier, \citet{arivazhagan2019federated} presented an algorithm to learn local heads and a global network body, but their local procedure jointly updates the head and body (using the same number of updates), and they did not provide any theoretical justification for their proposed method. Meanwhile, another line of work has studied federated learning in heterogeneous settings  \citep{karimireddy2020scaffold, wang2020tackling, pathak2020fedsplit, haddadpour2020federated, reddi2020adaptive,reisizadeh2020straggler,mitra2021achieving}, and the optimization-based insights from these works may be used to supplement our formulation and algorithm.


\textbf{Linear representation learning.} 
The idea to learn a shared representation of tasks is a classical approach in multi-task learning \citep{baxter2000model, bengio2013representation, lecun2015deep, ando2005framework, rish2008closed, pontil2013excess, balcan2015efficient,  denevi2018learning,  bullins2019generalize, tripuraneni2020theory, kong2020meta}. In particular, we aim to learn a low-dimensional subspace in which the ground-truth regressors for a collection of linear regression tasks lie. 
This problem is most similar to the linear representation learning problem considered by \cite{maurer2016benefit, tripuraneni2020provable, du2020fewshot}. 
All three of these works show statistical rates of convergence of solutions to the ERM objective to the ground-truth representation, with \cite{tripuraneni2020provable} and \cite{du2020fewshot} improving the $\mathcal{O}({\nicefrac{d}{n}})$ rate from \cite{maurer2016benefit} (in the realizable case) to $\mathcal{O}({\nicefrac{d}{mn}})$. \cite{du2020fewshot} also provide similar complexity-based  results for learning nonlinear representations with access to an ERM oracle, but their results in the linear case require $m=\Omega(d)$ samples per task, mitigating the benefit of cooperation.  
\cite{tripuraneni2020provable}  further present and analyze a Method-of-Moments-based algorithm to solve the ERM problem, which achieves sample complexity per task with efficient dimension-dependence (${\Theta}(\nicefrac{d}{n})$) but requires $m = {\Omega}(\nicefrac{1}{n\epsilon^2})$ samples per task to find an $\epsilon$-accurate representation. In contrast, we show that alternating minimization-descent requires only $m = \Omega((\nicefrac{d}{n} + \log(n))\log(\nicefrac{1}{\epsilon}))$ samples per client to obtain a representation with  $\epsilon$-accuracy.



\section{Problem Formulation}
The generic form of federated learning with $n$ clients is
\begin{equation}
    \min_{(q_1,\dots,q_n) \in \mathcal{Q}_n} \frac{1}{n}\sum_{i=1}^n  f_i(q_i) \label{obj1},
\end{equation}
where $f_i$ and $q_i$ are the error function and learning model for the $i$-th client, respectively, and $\mathcal{Q}_n$ is the space of feasible sets of $n$ models. 
We consider a supervised setting in which the data for the $i$-th client is generated by a distribution $(\mathbf{x}_i,y_i) \sim \mathcal{D}_i$. The learning model $q_i : \mathbb{R}^d \rightarrow \mathcal{Y}$ maps inputs $\mathbf{x}_i \in \mathbb{R}^d$ to predicted labels $q_i(\mathbf{x}_i)  \in \mathcal{Y}$, which we would like to resemble the true labels $y_i$.
 The error $f_i$ is in the form of an expected risk over $\mathcal{D}_i$, namely
    $f_i(q_i) \coloneqq \mathbb{E}_{(\mathbf{x}_i,y_i)\sim \mathcal{D}_i} [\ell(q_i(\mathbf{x}_i),y_i)]$,
where $\ell:\mathcal{Y}\times\mathcal{Y}\rightarrow\mathbb{R}$ is a loss function that penalizes the distance of $q_i(\mathbf{x}_i)$ from~$y_i$. 

In order to minimize $f_i$, the $i$-th client accesses a dataset of $M_i$ labeled samples $\{ (\mathbf{x}_i^j, y_i^j)\}_{j=1}^{M_i}$ from $\mathcal{D}_i$ for training. 
Federated learning addresses settings in which the $M_i$'s are typically small relative to the problem dimension while the number of clients $n$ is large. Thus, clients may not be able to obtain solutions $q_i$ with small expected risk by training completely locally on {\em only} their $M_i$ local samples. Instead, federated learning enables the clients to cooperate,
by exchanging messages with a central server, in order to  learn models using the cumulative data of all the clients.  



Standard approaches to federated learning aim at learning a {\em single} shared model $q=q_1=\dots=q_n$ that performs well on average across the clients \citep{mcmahan2017communication, li2018federated}. In this way, the clients aim to solve a special version of Problem~\eqref{obj1}, which is to minimize $(1/n) \sum_i f_i(q)$ over the choice of the shared model $q$.  However, this approach may yield a solution that performs poorly in heterogeneous settings where the data distributions $\mathcal{D}_i$ vary across the clients. Indeed, in the presence of data heterogeneity, the error functions $f_i$ will have different forms and their minimizers are not the same. Hence, learning a shared model $q$ may not provide good solution to  Problem~\eqref{obj1}. This necessities the search for more personalized solutions $\{q_i\}$ that can be learned in a federated manner using the clients' data. 

\noindent \textbf{Learning a Common Representation.} We are motivated by insights from centralized machine learning that suggest that heterogeneous data distributed across tasks may share a common representation despite having different labels \citep{bengio2013representation, lecun2015deep}; e.g., shared features across many types of images, or across word-prediction tasks.  
Using this common (low-dimensional) representation, the labels for each client can be simply learned using a linear classifier or a shallow neural network.

Formally, we consider a setting consisting of a global representation $\phi: \mathbb{R}^d \rightarrow \mathbb{R}^k$, which maps data points to a lower space of size $k$, and client-specific heads $h_i : \mathbb{R}^k \rightarrow \mathcal{Y}$.  The model for the $i$-th client is the composition of the client's local parameters and the representation: $q_i(\mathbf{x}) = (h_i \circ \phi)(\mathbf{x})$. 
Critically, $k \ll d$, meaning that the \textcolor{black}{number of parameters} that must be learned locally by each client is small. Thus, we can assume that any client's optimal classifier for any {\em fixed representation} is \textcolor{black}{easy} to compute, which motivates the following re-written global objective:
\begin{align}
    \min_{\phi \in \Phi} \ \frac{1}{n} \sum_{i=1}^n  \min_{h_i \in \mathcal{H}}\ f_i(h_i \circ \phi),  \label{obj}
\end{align}
where $\Phi$ is the class of feasible representations and $\mathcal{H}$ is the class of feasible heads. In our proposed scheme, clients cooperate to learn the global model using all clients' data, while they use their local information to learn their personalized head. We discuss this in detail in Section~\ref{sec:alg}.


\begin{figure}[t]
\begin{center}
\centerline{\includegraphics[width=0.5\columnwidth]{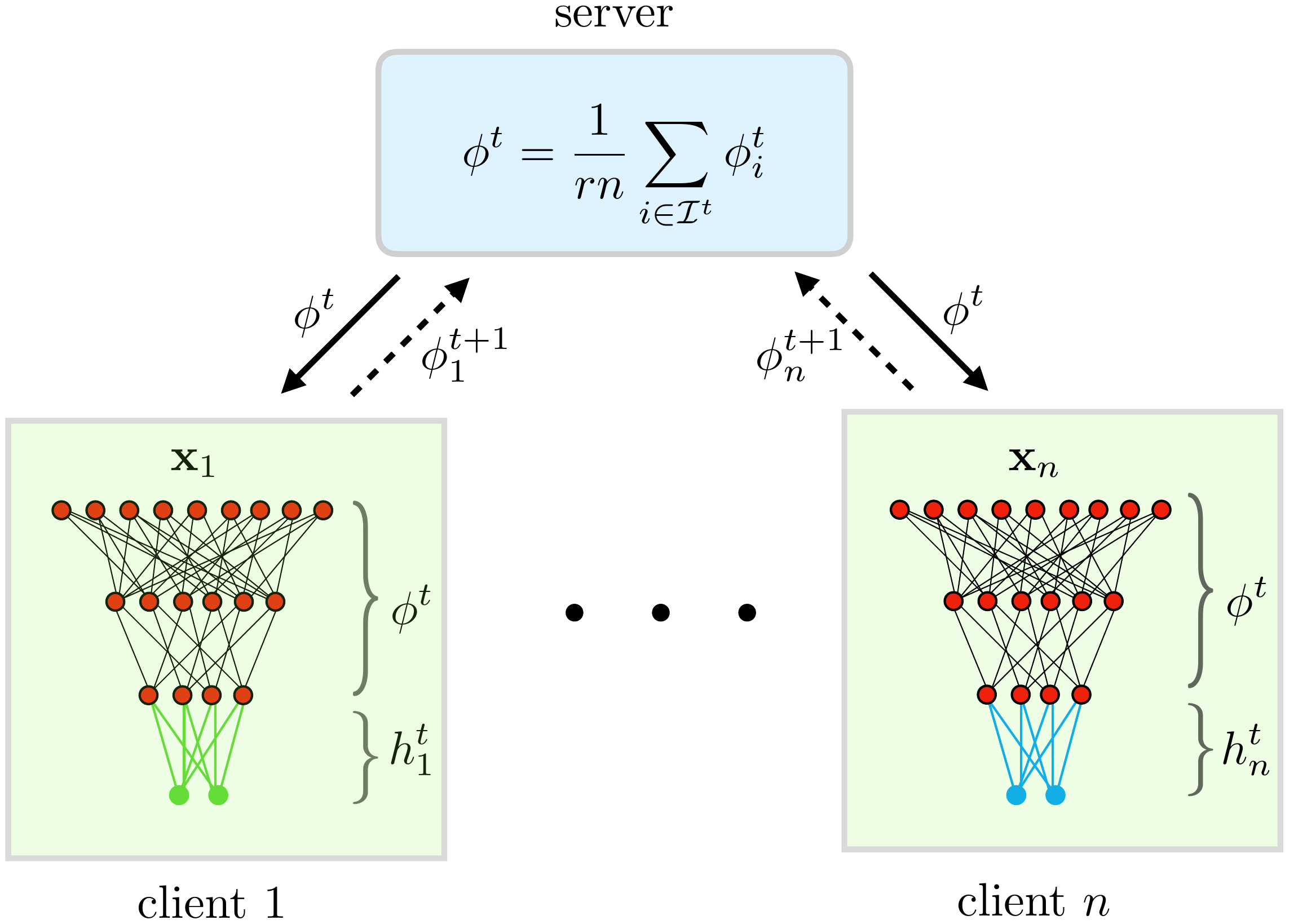}}
\vspace{-2mm}
\caption{Federated representation learning structure where clients and the server aim at learning a global representation $\phi$ together, while each client $i$ learns its unique head $h_i$ locally. }
\label{fig:fedrep2}
\end{center}
\vskip -0.3in
\end{figure}


\subsection{Comparison with Standard Federated Learning} \label{sec:pf_example}

{To formally demonstrate the advantage of our formulation over the standard (single-model) federated learning formulation in heterogeneous settings with a shared representation, we study a linear representation setting with  quadratic loss. As we will see below,  standard federated learning {\em cannot recover the underlying representation in the face of heterogeneity}, while our formulation does indeed recover it.}


 Consider a setting in which the functions $f_i$ are quadratic losses, the representation $\phi$ is a projection onto a $k$-dimensional subspace of $\mathbb{R}^d$ given by matrix $\mathbf{B} \in \mathbb{R}^{d \times k}$, and the $i$-th client's local head $h_i$ is a vector $\mathbf{w}_i \in \mathbb{R}^k$. In this setting, we model the local data of clients $\{\mathcal{D}_i\}_i$ such that
$ y_i =  {\mathbf{w}_i^\ast}^\top {\mathbf{B}^\ast}^\top\mathbf{x}_i$
for some ground-truth representation $\mathbf{B}^\ast \in \mathbb{R}^{d\times k}$ and local heads $\mathbf{w}_i^\ast \in \mathbb{R}^k$.
 This setting will be described in detail in Section~\ref{sec:linear}. In particular, one can show that the expected error over the data distribution $\mathcal{D}_i$ has the following form:
$
    f_i(\mathbf{w}_i \circ \mathbf{B}) \coloneqq \frac{1}{2}\|\mathbf{B} \mathbf{w}_i - \mathbf{B}^\ast \mathbf{w}^\ast_i\|_2^2
$. Consequently, Problem~\eqref{obj} becomes
\begin{align}
\min_{\mathbf{B} \in \mathbb{R}^{d \times k}, \mathbf{w}_i,\dots,\mathbf{w}_n \in\mathbb{R}^k} \frac{1}{2n}\sum_{i=1}^n\|\mathbf{B} \mathbf{w}_i - \mathbf{B}^\ast \mathbf{w}^\ast_i\|_2^2\label{us}.
\end{align}
In contrast, standard federated learning methods, which aim to learn a shared model $(\mathbf{B}, \mathbf{w})$ for all the clients, solve
\begin{align}
\min_{\mathbf{B} \in \mathbb{R}^{d \times k}, \mathbf{w} \in\mathbb{R}^k}\frac{1}{2n}\sum_{i=1}^n\|\mathbf{B} \mathbf{w} - \mathbf{B}^\ast \mathbf{w}^\ast_i\|_2^2. \label{them}
\end{align}
Let $(\hat{\mathbf{B}}, \{\hat{{\mathbf{w}}_i}\}_i) 
$  denote a global minimizer of  
\eqref{us}.  We thus have
$\hat{\mathbf{B}} \hat{\mathbf{w}}_i = \mathbf{B}^\ast \mathbf{w}^\ast_i$ for all $i\in [n]$. Also, it is not hard to see that $({\mathbf{B}}^\diamond, \mathbf{w}^\diamond) 
$ is a global minimizer of \eqref{them} if and only if ${\mathbf{B}}^\diamond \mathbf{w}^\diamond = {\mathbf{B}}^\ast (\frac{1}{n}\sum_{i=1}^n \mathbf{w}^\ast_i)
$. Thus, our formulation finds an exact solution with zero global error, whereas standard federated learning has global error of 
$
 \frac{1}{2n}\sum_{i=1}^n \| \frac{1}{n} \mathbf{B}^\ast\sum_{i'=1}^n (\mathbf{w}^\ast_{i'} - \mathbf{w}^\ast_i)\|_2^2
$, which grows with the heterogeneity of the $\mathbf{w}_i^\ast$. Moreover, since solving our formulation provides $n$ matrix equations, we can fully recover the column space of $\mathbf{B}^\ast$ as long as $\mathbf{w}_i^*$'s span $\mathbb{R}^k$. In contrast, solving \eqref{them} yields only one matrix equation, so there is no hope to recover the column space of $\mathbf{B}^\ast$ for any $k>1$.

\section{FedRep Algorithm} \label{sec:alg}
FedRep solves Problem \eqref{obj} by distributing the computation across clients. The server and clients aim to learn the parameters of the global representation together, while the $i$-th client aims to learn its unique local head 
locally (see Figure~\ref{fig:fedrep2}).  To do so, FedRep alternates between client updates and a server update on each communication round. 

\textbf{Client Update.} On each round, a constant fraction $r\in(0,1]$ of the clients are selected to execute a client update. In the client update, client $i$ makes $\tau_h$ local gradient-based updates to solve for its optimal head given the current global representation $\phi^t$ communicated by the server. Namely, for $s = 1,\dots,\tau_h$, client $i$ updates its head as follows:
$$h_i^{t,s} = \texttt{GRD}( f_i(h_i^{t,s-1}, \phi^t), h_i^{t,s-1},\alpha),$$ where $\texttt{GRD}(f, h, \alpha)$ is generic notation for an update of the variable $h$ using a gradient of function $f$ with respect to $h$ and the step size $\alpha$. For example, $\texttt{GRD}( f_i(h_i^{t,s-1},\phi^t), h_i^{t,s-1},\alpha)$ can be a step of gradient descent, stochastic gradient descent (SGD), SGD with momentum, etc. Typically, we will choose $\tau_h$ to be large, since more local epochs for the head means that we come closer to solving the inner minimization in \eqref{obj}, which means that the updates for the representation are more accurate. 

Next, the client executes $\tau_\phi$ local updates for its representation, starting from the global representation $\phi^{t-1}$: $$\phi_i^{t,s} = \texttt{GRD}( f_i(h_i^{t,\tau_h}, \phi_i^{t,s-1}), \phi_i^{t,s-1},\alpha),$$ for $s=1,\dots,\tau_\phi$. 

\textbf{Server Update.} Once the local updates with respect to the head and representation finish, the client participates in the server update by sending its locally-updated representation $\phi_{i}^{t,\tau_\phi}$
to the server. The server then averages the local updates to compute the next representation $\phi^{t}$. The entire procedure is outlined in Algorithm \ref{alg:general}.

\begin{algorithm}[tb] 
\caption{\texttt{FedRep}}
\begin{algorithmic} \label{alg:general}
\STATE \textbf{Parameters:} Participation rate $r$, step size $\alpha$; number of local updates for the head $\tau_h$ and for the representation $\tau_\phi$; number of communication rounds $T$.
\STATE Initialize $\phi^0, h_1^0, \dots, h_n^0$
 \FOR{$t=1 ,2, \dots, T$}
  \STATE Server receives a batch of clients $\mathcal{I}^t$ of size $rn$\;
  \STATE Server sends current representation $\phi^t$ to these clients\;
  \FOR{\textbf{each} client $i$ in $\mathcal{I}^t$}
  \STATE Client $i$ initializes $h_i^{t,0} \leftarrow h_i^{t-1,\tau_h}$\;
  \STATE Client $i$ makes $\tau_h$ updates to its head:\;
  \FOR{$s = 1$ {\bfseries to} $\tau_h$}
\STATE $h_i^{t,s} \leftarrow  \texttt{GRD}( f_i(h_i^{t,s-1},\phi^{t-1}), h_i^{t,s-1},\alpha)$ \;
  \ENDFOR
\STATE Client $i$ initializes $\phi_i^{t,0} \leftarrow \phi^{t-1}$\;
  \STATE Client $i$ makes $\tau_\phi$ updates to its representation:\;
\FOR{$s = 1$ {\bfseries to} $\tau_\phi$}
\STATE $\phi_i^{t,s} \leftarrow  \texttt{GRD}( f_i(h_i^{t,\tau_h},\phi_i^{t,s-1}), \phi_i^{t,s-1},\alpha)$ \;
  \ENDFOR
    \STATE Client $i$ sends updated representation $\phi_i^{t,\tau_\phi}$ to server \;
  \ENDFOR
  \FOR{\textbf{each} client $i$ not in $\mathcal{I}^t$,} 
  \STATE Set $h_i^{t,\tau_h} \leftarrow h_i^{t-1,\tau_h}$\;
  \ENDFOR
  \STATE Server computes the new representation as \\ $\quad \quad \phi^{t}=\frac{1}{rn} \sum_{i\in \mathcal{I}^t} \phi_{i}^{t,\tau_\phi}$ \;
 \ENDFOR
 \end{algorithmic}
\end{algorithm}

\section{Low-Dimensional Linear Representation} \label{sec:linear}

In this section, we analyze an instance of Problem~\eqref{obj} with quadratic loss functions and linear models, as discussed in Section \ref{sec:pf_example}. Here, each client's problem is to solve a linear regression with a two-layer linear neural network. 
In particular, each client $i$ attempts to find a shared global projection onto a low-dimension subspace $\mathbf{B} \in \mathbb{R}^{d\times k}$ and a unique regressor $\mathbf{w}_i \in \mathbb{R}^{k}$ that together accurately map its samples $\mathbf{x}_i \in \mathbb{R}^d$ to labels $y_i \in \mathbb{R}$. The matrix $\mathbf{B}$ corresponds to the representation $\phi$, and $\mathbf{w}_i$ corresponds to local head $h_i$ for the $i$-th client. We thus have $(q_{h_i} \circ q_\phi)(\mathbf{x}_i) =  \mathbf{w}_i^\top  \mathbf{B}^\top \mathbf{x}_i$. Hence, the loss function for client $i$ is given by:
\begin{equation} \label{obj:client_exp}
     f_i(\mathbf{w}_i, \mathbf{B}) \coloneqq  \tfrac{1}{2}\mathbb{E}_{(\mathbf{x}_i,y_i) \sim \mathcal{D}_i}\left[({y}_i - \mathbf{w}_i^\top  \mathbf{B}^\top \mathbf{x}_i)^2 \right]
\end{equation}
meaning that the global objective is:
\begin{equation} \label{obj:glob_linear}
    \min_{\substack{\mathbf{B}\in \mathbb{R}^{d \times k}\\ \mathbf{W} \in \mathbb{R}^{n\times k}}}\! F(\mathbf{B},\mathbf{W}) \!\coloneqq\!   \frac{1}{2n}\sum_{i=1}^n  \mathbb{E}_{(\mathbf{x}_i,y_i)}\left[({y}_i - \mathbf{w}_i^\top  \mathbf{B}^\top \mathbf{x}_i)^2 \right]\!, 
\end{equation}
where $\mathbf{W} = [\mathbf{w}_1^\top,\dots,\mathbf{w}_n^\top] \in  \mathbb{R}^{n\times k}$ is the concatenation of client-specific heads. To evaluate the ability of FedRep to learn an accurate representation,
we model the local datasets $\{\mathcal{D}_i\}_i$ such that, for $i=1\dots,n$
$$
y_i =  {\mathbf{w}_i^\ast}^\top {\mathbf{B}^\ast}^\top\mathbf{x}_i,
$$
for some ground-truth representation $\mathbf{B}^\ast \in \mathbb{R}^{d\times k}$ and local heads $\mathbf{w}_i^\ast \in \mathbb{R}^k$--i.e. a standard regression setting. In other words, all of the clients' optimal solutions live in the same $k$-dimensional subspace of $\mathbb{R}^d$, where $k$ is assumed to be small.  
Moreover, we make the following standard assumption on the samples $\mathbf{x}_i$.
\begin{assumption}[Sub-gaussian design] \label{assump:subgauss} The samples $\mathbf{x}_{i} \in \mathbb{R}^d$ are i.i.d. with mean $\mathbf{0}$, covariance $\mathbf{I}_d$, and are $\mathbf{I}_d$-sub-gaussian, i.e. $\mathbb{E}[e^{ {\mathbf{v}^\top \mathbf{x}_{i}}}] \leq e^{\| \mathbf{v}\|_2^2/2}$ for all $\mathbf{v}\in \mathbb{R}^d$.
\end{assumption}


\subsection{FedRep} \label{sec:linear_fedrep}

We next discuss how FedRep tries to recover the optimal representation in this setting. First, the server and clients execute the Method of Moments to learn an initial representation. Then, client and server updates are executed in an alternating fashion as follows.

\textbf{Client Update.} As in Algorithm \ref{alg:general}, $rn$ clients are selected on round $t$ to update their current local head $\mathbf{w}_i^t$ and the global representation $\mathbf{B}^t$. Each selected client $i$ samples a fresh batch $\{\mathbf{x}_i^{t,j}, y_i^{t,j}\}_{j=1}^m$ of $m$ samples according to its local data distribution $\mathcal{D}_i$ to use for updating both its head and representation on each round $t$ that it is selected. That is, within the round,  client $i$ considers the batch loss 
\begin{equation} \label{obj:client}
      \hat{f}^t_i(\mathbf{w}_i^t, \mathbf{B}^t) \coloneqq  \frac{1}{2m}\sum_{j=1}^m ({y}_i^{t,j} - \mathbf{w}_i^{t^\top}  \mathbf{B}^{t^\top} \mathbf{x}_i^{t,j})^2. 
\end{equation}
Since $\hat{f}^t_i$ is strongly convex with respect to $\mathbf{w}_i^t$, the client can find an update for a local head that is $\epsilon$-close to the global minimizer of \eqref{obj:client} after at most $\log (1/\epsilon)$ local gradient updates. Alternatively, since the function is also quadratic, the client can solve for the optimal $\mathbf{w}$ directly in only $\mathcal{O}(mk^2 + k^3)$ operations. Thus, since FedRep calls for many local updates for the head, to simplify the analysis we assume each selected client obtains $\mathbf{w}_i^{t+1} = \argmin_\mathbf{w} \hat{f}^t_i(\mathbf{w}, \mathbf{B}^t)$ during each round of local updates.  

\textbf{Server Update.} After updating its head, client $i$ updates the global representation with one step of gradient descent using the same $m$ samples and sends the update to the server, as outlined in Algorithm \ref{alg:2}. Note that in practice, each client may execute multiple gradient-based updates before sending its updated representation back to the server, but here we consider the case that they make one step of gradient descent for simplicity. Once the server receives the representations, it averages them and orthogonalizes the resulting matrix to compute the new representation. 



\subsection{Analysis}
\label{sec:fed-analysis}

As mentioned earlier, in FedRep, each client $i$ perform an alternating minimization-descent method to solve its nonconvex objective in \eqref{obj:client}. This means the global loss over all clients at round $t$ is given by 
\begin{equation} \label{obj:global}
     \frac{1}{n}\sum_{i=1}^n \hat{f}^t_i(\mathbf{w}_i^t, \mathbf{B}^t) \coloneqq \frac{1}{2mn}\sum_{i=1}^n \sum_{j=1}^m ({y}_i^{t,j} - \mathbf{w}_i^{t^\top}  \mathbf{B}^{t^\top} \mathbf{x}_i^{t,j})^2. 
\end{equation}
This objective has many global minima, including all pairs of matrices $(\mathbf{Q}^{-1}\mathbf{W}^*, \mathbf{B}^* \mathbf{Q}^\top)$ where $\mathbf{Q}\in \mathbb{R}^{k \times k}$ is invertible, eliminating the possibility of exactly recovering the ground-truth factors $(\mathbf{W}^*, \mathbf{B}^*)$. Instead, the ultimate goal of the server is to
recover the ground-truth {\em representation}, i.e.,  the column space of $\mathbf{B}^*$.
To evaluate how closely the column space is recovered, we define the distance between subspaces as follows.


\begin{algorithm}[tb] 
\caption{\texttt{FedRep} for linear regression}
\begin{algorithmic} \label{alg:2}
\STATE \textbf{Input:} Step size $\eta$; number of rounds $T$, participation rate $r$.\;
 \STATE \textbf{Initialization:}  Each client $i\!\in\![n]$ sends $\mathbf{Z}_i\! \coloneqq\!\frac{1}{m}\sum_{j=1}^m (y_i^{0,j})^2 \mathbf{x}_i^{0,j} (\mathbf{x}_i^{0,j})^\top$ to server, server computes \;
   \vspace{-2mm}
 \STATE $$\quad \quad \quad \mathbf{U} \mathbf{D}\mathbf{U}^\top\!\leftarrow\! \text{rank-}k \text{ SVD}(\tfrac{1}{n}\textstyle{\sum_{i=1}^n\mathbf{Z}_i)}$$
   \vspace{-6mm}
 \STATE Server initializes $\mathbf{B}^0 \leftarrow \mathbf{U}$
 \FOR{$t=1 ,2, \dots, T$}
 \STATE Server receives a subset $\mathcal{I}^t$ of clients of size $rn$ \;
  \STATE Server sends current representation $\mathbf{B}^t$ to these clients\;
  \FOR{$i\in\mathcal{I}^t $} 
  \STATE \textbf{Client update:} \;
  \STATE Client $i$ samples a fresh batch of $m$ samples
  \STATE Client $i$ updates $\mathbf{w}_i$: \;
  \STATE $\quad\quad \mathbf{w}_i^{t+1} \gets \argmin_{\mathbf{w}} \hat{f}_i^t(\mathbf{w}, \mathbf{B}^t)$
  \STATE \;
  \vspace{-3mm}
  \STATE Client $i$ updates representation: \; 
  \STATE $
      \quad \quad \mathbf{B}_i^{t+1} \gets  \mathbf{B}^t - \eta \nabla_{\mathbf{B}}\hat{f}_i^{t}(\mathbf{w}_i^{t+1}, \mathbf{B}^t) $ \;
      \STATE \;
      \vspace{-3mm}
      \STATE Client $i$ sends $\mathbf{B}_i^{t+1}$ to the server \;
  \ENDFOR
  \STATE \textbf{Server update:} $\mathbf{\bar{B}}^{t+1} \gets \frac{1}{rn}\sum_{i \in \mathcal{I}^t} \mathbf{B}^{t+1}_i; \quad \mathbf{{B}}^{t+1}, \mathbf{{R}}^{t+1} \gets \text{QR}(\mathbf{\bar{B}}^{t+1})$ \;
 \ENDFOR
 \end{algorithmic}
\end{algorithm}

\begin{definition} The principal angle distance between the column spaces of $\mathbf{B}_1, \mathbf{B}_2 \in \mathbb{R}^{d \times k}$ is given by
\begin{equation}
\dist(\mathbf{B}_1, \mathbf{B}_2) \coloneqq \|\mathbf{\hat{B}}_{1,\perp}^\top \mathbf{\hat{B}}_{2} \|_2, 
\end{equation}
where $\mathbf{\hat{B}}_{1,\perp}$ and  $\mathbf{\hat{B}}_{2}$ are orthonormal matrices satisfying 
$\text{\em span}(\mathbf{\hat{B}}_{1,\perp}) = \text{\em span}(\mathbf{{B}}_{1})^\perp$ and $\text{\em span}(\mathbf{\hat{B}}_{2}) = \text{\em span}(\mathbf{{B}}_{2}).$
\end{definition}

The principal angle distance is a typical metric for measuring the distance between subspaces (e.g. \cite{Jain_2013}). Next, we make two standard assumptions. 


\begin{assumption}[Client diversity] \label{assump:tasks}
Let $\bar{\sigma}_{\min, \ast} \coloneqq \min_{\mathcal{I} \in [n], |\mathcal{I}| = rn} \sigma_{\min}(\frac{1}{\sqrt{rn}} \mathbf{W}^\ast_{\mathcal{I}})$, i.e. $\bar{\sigma}_{\min, \ast}$ is the minimum singular value of any matrix that can be obtained by taking $rn$ rows of $\frac{1}{\sqrt{rn}}\mathbf{W}^\ast$.  Then $\bar{\sigma}_{\min,\ast}>0$. 
\end{assumption}
Assumption \ref{assump:tasks} states that if we select any $rn$ clients, their optimal heads span $\mathbb{R}^k$. Indeed, this assumption is weak as we expect the number of participating clients $rn$ to be substantially larger than $k$. Note that if we do not have client solutions that span $\mathbb{R}^k$,
recovering $\mathbf{B}^*$ would be impossible because the samples $(\mathbf{x}_i^j, y_i^j)$ may never contain any information about one or more features of $\mathbf{B}^*$.
\begin{assumption}[Client normalization] \label{assump:norm}
The ground-truth client-specific parameters satisfy $\|\mathbf{w}_{i}^{\ast}\|_2 = \sqrt{k}$ for all $i \in [n]$, and $\mathbf{B}^*$ has orthonormal columns.
\end{assumption}
Assumption 2 ensures that the ground-truth matrix $\mathbf{W}^* {\mathbf{B}^*}^\top$ is row-wise {\em incoherent}, i.e. its row norms have similar magnitudes. We define this formally in Appendix \ref{app:proof_main}.
Incoherence of the ground-truth matrices is a key property required for efficient matrix completion and other sensing problems with sparse measurements \citep{chi2019nonconvex}. Since our measurement matrices are row-wise sparse, we require the row-wise incoherence of the ground truth. Note that Assumption \ref{assump:norm} can be relaxed to allow $\|\mathbf{w}_i^{\ast}\|_2 \leq O(\sqrt{k})$, as the exact normalization is only for simplicity of analysis.






Our main result shows that the iterates $\{\mathbf{B}^t\}_t$ generated by FedRep in this setting linearly converge to the optimal representation $\mathbf{B}^\ast$ in principal angle distance.

\begin{customthm}{1} \label{thrm:linear} 
Define $E_0 \coloneqq 1- \dist^2(\mathbf{{B}}^0,\mathbf{{B}}^\ast)$ and $\bar{\sigma}_{\max, \ast} \coloneqq \max_{\mathcal{I} \in [n], |\mathcal{I}| = rn} \sigma_{\max} (\frac{1}{\sqrt{rn}}\mathbf{W}^\ast_{\mathcal{I}})$ and $\bar{\sigma}_{\min, \ast} \coloneqq \min_{\mathcal{I} \in [n], |\mathcal{I}| = rn} \sigma_{\min}(\frac{1}{\sqrt{rn}} \mathbf{W}^\ast_{\mathcal{I}})$, i.e. the maximum and minimum singular values of any matrix that can be obtained by taking $rn$ rows of $\frac{1}{\sqrt{rn}}\mathbf{W}^\ast$. Let $\kappa\!\coloneqq \! \bar{\sigma}_{\max, \ast}/\bar{\sigma}_{\min, \ast}$.
Suppose that $m \geq c( \kappa^4 k^2 d/(E_0^2 rn) + \kappa^4 k^3\log(rn)/E_0^2)$ for some absolute constant $c$. 
Then for any $t$ and any $\eta \leq 1/(4  \bar{\sigma}_{\max,\ast}^2)$, we have 
\begin{align}
    \dist(\mathbf{{B}}^T, \mathbf{{B}}^\ast)
    &\leq \left(1 -  \eta E_0\bar{\sigma}_{\min,\ast}^2 /2\right)^{T/2}\; \dist(\mathbf{{B}}^0,\mathbf{{B}}^\ast), \label{result}
\end{align}
with probability at least $1 - Te^{-100\min(k^2\log(rn),d)}$.
\end{customthm}

From Assumption \ref{assump:tasks}, we have that $\bar{\sigma}^2_{\min, *} > 0$, so the RHS of \eqref{result} strictly decreases with $T$ for appropriate step size. 
Considering the complexity of $m$ and the fact that the algorithm converges exponentially fast, the total number of samples required per client to reach an $\epsilon$-accurate solution in principal angle distance is $\Theta\left(m\log \left(\nicefrac{1}{\epsilon}\right)\right)$, which is
\begin{align}
\Theta\left(\left[ \kappa^4 k^2 \left(\nicefrac{d}{rn}+ k\log  (rn) \right)\right]\log \left(\nicefrac{1}{\epsilon}\right) \right).\label{complex}\end{align}
Next, a few remarks about this sample complexity follow. 

\textbf{When and whom does federation help?} Observe that for a single client with no collaboration, the sample complexity scales as $\Theta(d).$ With FedRep, however, the sample complexity scales as $\Theta(  \nicefrac{d}{n} + \log(n) )$, treating $k, \kappa$ and $r$ as constants. Thus, so long as $ \nicefrac{d}{n}+ \log(n) \ll d,$ federation helps. This holds in several settings, for instance when  $1 \ll n \ll e^{\Theta(d)}.$ In practical scenarios, $d$ (the data dimension) is large, and thus $e^{\Theta(d)}$ is exponentially larger; thus collaboration helps {\em each individual client.} 
Furthermore, new clients who enter the system later have a representation available for free, so these new clients' sample complexity is only $\Theta(k)$ because they each  only need to solve a $k$-dimensional linear regression problem \citep{hsu2012random}. Thus, both the overall system benefits (a representation has been learned, which is useful for the new client because it now only needs to learn a head), and each individual client that took part in the federated training also benefits.

\textbf{Connection to matrix sensing.} 
The problem in  \eqref{obj:glob_linear} is an instance of matrix sensing; {see the proof in Appendix \ref{app:proof_main} for more details}. Considering this connection, our theoretical results also contribute to the theoretical study of matrix sensing. Although matrix sensing is a well-studied problem, our setting presents two new analytical challenges: (i) due to row-wise sparsity in the measurements, the sensing operator does not satisfy the commonly-used Restricted Isometry Property (RIP) within an efficient number of samples, i.e., it does not efficiently concentrate to an identity operation on all rank-$k$ matrices, and (ii) FedRep executes a novel non-symmetric procedure. We further discuss these challenges in Appendix~\ref{app:challenges}. To the best of our knowledge, Theorem~1 provides the first convergence result for an alternating minimization-descent procedure to solve a matrix sensing problem. It is also the first result to show sample-efficient linear convergence of any solution to a matrix sensing with rank-one, row-wise sparse measurements.
The state-of-the-art result for the closest matrix sensing setting to ours is given by \citet{zhong2015efficient} for rank-1, independent Gaussian measurements, which our result matches up to an $\mathcal{O}(\kappa^2)$ factor. However, our setting is more challenging as we have rank-1 {\em and} row-wise sparse measurements, and dependence on $\kappa^4$ has been previously observed in settings with sparse measurements, e.g. matrix completion \citep{Jain_2013}. 


\textbf{Representation learning, dimensionality reduction and new users.}  
Theorem \ref{thrm:linear} 
concerns a linear representation learning setting that is of interest beyond federated learning to representation learning problems more broadly, such as in meta-learning and multi-task learning. 
This   setting has garnered significant attention recently in large part due to empirical evidence that representation learning can explain the success of meta-learning methods on few-shot learning tasks \citep{raghu2019rapid}. As shown by \cite{maurer2016benefit}, \cite{du2020fewshot} and \cite{tripuraneni2020provable},  learning an accurate $k$-dimensional representation during training (or meta-training) reduces the sample complexity of solving a new task from $\Theta(d)$ to $\Theta(k)$ in the linear case, enabling strong few-shot performance if $k$ is small. Theorem \ref{thrm:linear} shows that FedRep learns an accurate $k$-dimensional representation during training in the linear case, so these prior results imply that FedRep also needs only $\Theta(k)$ samples to learn the model for the new client.
Further, Theorem \ref{thrm:linear} shows that alternating minimization-descent (FedRep in the linear case) efficiently learns the representation compared to the methods studied in prior works (see Section \ref{sec:rw} for a detailed comparison).

\textbf{Remark on initialization.} Theorem \ref{thrm:linear} requires that the initial principal angle distance $\text{dist}(\mathbf{{B}}^0, \mathbf{{B}}^\ast)$ is bounded away from 1 by a constant. This can be efficiently achieved by the Method of Moments without increasing the sample complexity for each client up to log factors \citep{tripuraneni2020provable}. In turn, each user must send the server a polynomial of their data, namely $\sum_{j=1}^m (y_{i}^j)^2 \mathbf{x}_{i}^j (\mathbf{x}_i^j)^\top$ at the start of the learning procedure, which does not compromise privacy. We discuss the details of this in Appendix \ref{app:proof_main}.



\section{Experiments} \label{sec:experiments}
We focus on three points in our experiments: (i) the effect of many local updates for the local head in FedRep (ii) the quality of the global representation learned by FedRep and (iii) the applicability of FedRep to a wide range of datasets. Full experimental details are provided in Appendix \ref{app:experiments}.

\subsection{Synthetic Data}

We start by experimenting with an instance of the multi-linear regression problem analyzed in Section \ref{sec:linear}. Consistent 
with this formulation, we generate synthetic samples $\mathbf{x}_i^j \sim \mathcal{N}(0, \mathbf{I}_d)$ and labels $y_{i}^j \sim \mathcal{N}(\mathbf{w}_i^{\ast^\top}{\mathbf{B}}^{\ast^\top}\mathbf{x}_i^j, 10^{-3})$ (here we include an additive Gaussian noise). The ground-truth heads $\mathbf{w}_i^\ast \in \mathbb{R}^k$ for clients $i \in [n]$ and the ground-truth representation ${\mathbf{B}}^\ast \in \mathbb{R}^{d \times k}$ are generated randomly by sampling and normalizing Gaussian matrices. 


\begin{figure}[t]
\begin{center}
\centerline{\includegraphics[width=0.7\columnwidth]{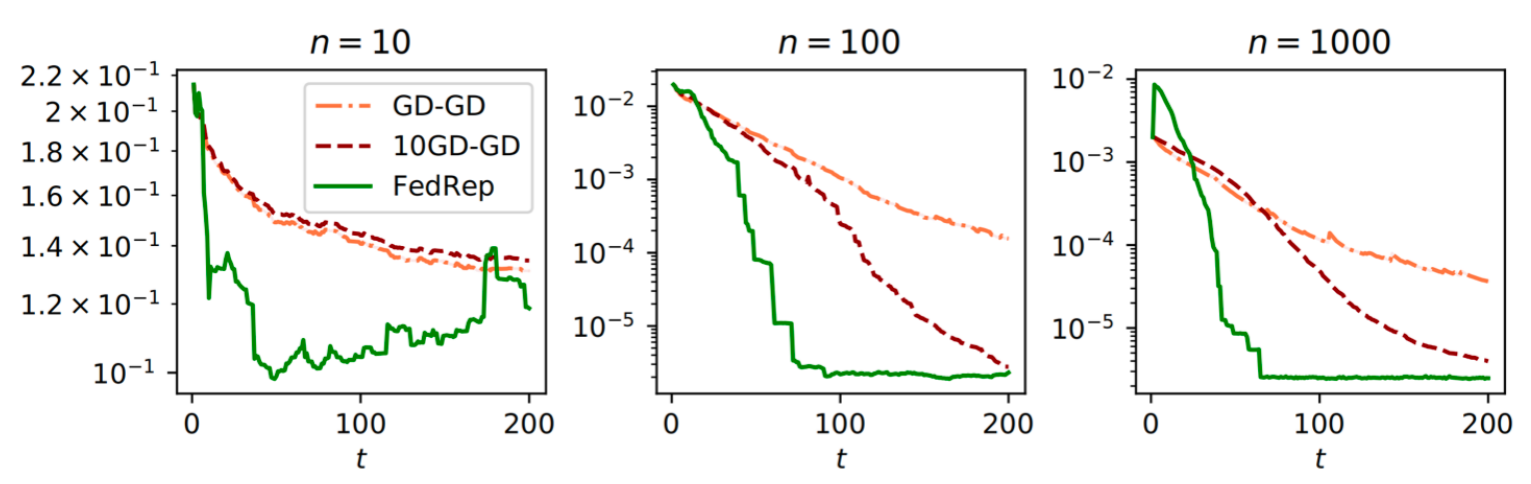}}
\vspace{-2mm}
\caption{Comparison of (principal angle) distances between the ground-truth and estimated representations by FedRep and alternating gradient descent algorithms for different numbers of clients $n$. In all plots, $d=10$, $k=2$, $m=5$, and $r=0.1$.}
\label{fig:synth11}
\end{center}
\vskip -0.3in
\end{figure}

\textbf{Benefit of finding the optimal head.}  We first demonstrate that the convergence of FedRep improves with larger number of clients $n$, making it highly applicable to federated settings. Further, we give evidence showing that this improvement is augmented by the minimization step in FedRep, since methods that replace the minimization step in FedRep with 1 and 10 steps of gradient descent (GD-GD and 10GD-GD, respectively) do not scale properly with $n$. 
In Figure \ref{fig:synth11}, we plot convergence trajectories for FedRep, GD-GD, and 10GD-GD for four different values of $n$ and fixed $m, d, k$ and $r$.  As we observe in Figure \ref{fig:synth11}, by increasing the number of nodes $n$, clients converge to the true representation faster. Also, running more local updates for finding the local head accelerates the convergence speed of FedRep. In particular, FedRep which exactly finds the optimal local head at each round has the fastest rate compared to GD-GD and 10GD-GD that only run 1 and 10 local updates, respectively, to learn the head.




\textbf{Generalization to new clients.} 
Next, we evaluate the effectiveness of the representation learned by FedRep in reducing the sample complexity for a new client which has not participated in training. We compare against FedSGD, which executes distributed SGD to learn a single model $(\mathbf{B},\mathbf{w})$. We first train FedRep and FedSGD on a fixed set of $n=100$ clients as in Figure~\ref{fig:fedrep1},  where $(d,k)\!=\!(20,2)$. 
The new client has access to $m_{\text{new}}$ labeled local samples. It will use the representation $\mathbf{{B}}^*\in \mathbb{R}^{d \times k}$ learned from the training clients, and learns a personalized head using this representation and its local training samples. For both FedRep and FedSGD, we solve for the optimal head given these samples and the representation learned during training. We compare the MSE of the resulting model on the new client's test data to that of a model trained by only using the $m_{\text{new}}$ labeled samples from the new client (Local Only) in Figure \ref{fig:bar_new}. The large error for FedSGD demonstrates that it does not learn the ground-truth representation. Meanwhile, the representation learned by FedRep allows an accurate model to be found for the new client as long as $m_{\text{new}}\geq k$, which drastically improves over the complexity for Local Only ($m_{\text{new}}\!=\!\Omega(d)$).



\begin{figure}[t]
\begin{center}
\centerline{\includegraphics[width=0.53\columnwidth]{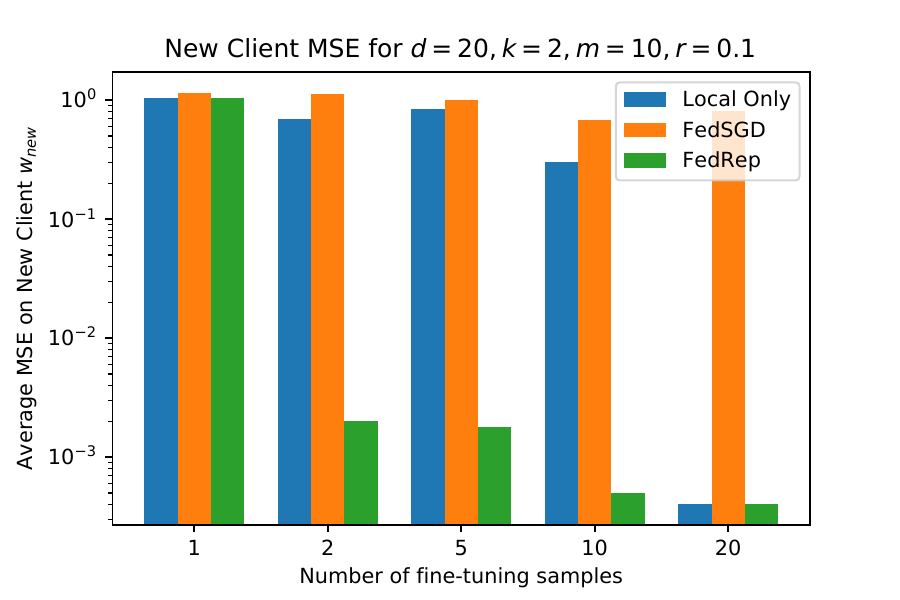}}
\vspace{-2mm}
\caption{MSE on new clients sharing the representation after fine-tuning using various numbers of samples from the new client.}
\label{fig:bar_new}
\end{center}
\end{figure}

\subsection{Real Data}
We next investigate whether these insights apply to nonlinear models and real datasets.

\textbf{Datasets and Models.} We use four real datasets: CIFAR10 and CIFAR100 \citep{krizhevsky2009learning}, FEMNIST \citep{caldas2018leaf,cohen2017emnist} and Sent140 \citep{caldas2018leaf}. The first three are image datasets and the last is a text dataset for which the goal is to classify the sentiment of a tweet as positive or negative. 
 We control the heterogeneity of CIFAR10 and CIFAR100 by assigning different numbers $S$ of classes per client, from among 10 and 100 total classes, respectively.  Each client is assigned the same number of training samples, namely $50000/n$. For FEMNIST, we restrict the dataset to 10 handwritten letters and assign samples to clients according to a log-normal distribution as in \cite{li2019feddane}. We consider a partition of $n\!=\!150$ clients with an average of 148 samples/client. For Sent140, we use the natural assignment of tweets to their author, and use $n\!=\!183$ clients with an average of 72 samples per client. We use 5-layer CNNs for the CIFAR datasets, a 2-layer MLP for FEMNIST, and an RNN for Sent$140$ \textcolor{black}{(details provided in Appendix \ref{app:experiments})}.

\textbf{Baselines.} We compare against a variety of personalized federated learning techniques as well as methods for learning a single global model and their fine-tuned analogues. Among the personalized methods, FedPer \citep{arivazhagan2019federated} is most similar to ours, as it also learns a global representation and personalized heads, but makes simultaneous local updates for both sets of parameters, therefore makes the same number of local updates for the head and the representation on each local round. Fed-MTL \citep{smith2017federated} learns local models and a regularizer to encode relationships among the clients, PerFedAvg \citep{fallah2020personalized} leverages meta-learning  to learn a single model that performs well after adaptation on each task, and LG-FedAvg \citep{liang2020think} learns local representations and a global head. APFL \citep{deng2020adaptive} interpolates between local and global models, and L2GD \citep{hanzely2020federated} and Ditto \citep{li2020ditto} learn local models that are encouraged to be close together by global regularization. For global FL methods, we consider FedAvg \citep{mcmahan2017communication}, SCAFFOLD \citep{karimireddy2020scaffold}, and FedProx \citep{li2018federated}. To obtain fine-tuning results, we first train the global model for the full training period, then each client then fine-tunes only the head on its local training data for 10 epochs of SGD before computing the final test accuracy. 

\begin{table*}[ht]
\caption{Average test accuracies on various partitions of CIFAR10, CIFAR100, Sent140 and FEMNIST with participation rate $r\!=\!0.1$.
}
\vskip -1in
\begin{center}
\footnotesize
\begin{tabular}{l|ccc|cc|c|c}
\toprule\label{bigtable1}
& \multicolumn{3}{c|}{{CIFAR10}}  &  \multicolumn{2}{c|}{{CIFAR100}}       & Sent140 & FEMNIST  \\
\midrule
(\# clients $n$, \# classes per client $S$) & $(100,2)\!$ & $\!(100,5)$ & $(1000,2)$ & $(100,5)$  & $(100,20)$ & $(183,2)$ & $(150,3)$ \\
\midrule
Local Only & $\mathbf{89.79}$ & 70.68 & 78.30 & 75.29 & 41.29 & 69.88 &  60.86 \\
\midrule
FedAvg \citep{mcmahan2017communication} &  42.65 & 51.78 & 44.31 & 23.94 & 31.97 & 52.75 & 51.64 \\ 
FedAvg+FT & 87.65 & 73.68  & 82.04 & $\mathbf{79.34}$ & 55.44  &71.92  & 72.41\\ 
FedProx \citep{li2018federated} & 39.92 & 50.99 & 21.93  & 20.17 & 28.52 & 52.33 &  18.89 \\
FedProx+FT & 85.81 & 72.75 & 75.41 & 78.52 & 55.09 & 71.21 & 53.54 \\
SCAFFOLD \citep{karimireddy2020scaffold} & $37.72$ & $47.33$ & $33.79$ &  20.32 & 22.52 & $51.31$ & 17.65 \\
SCAFFOLD+FT  & $86.35$ &  $68.23$ &$78.24$ & 78.88 & 44.34 & 71.49  & 52.11  \\
\midrule 
Fed-MTL \citep{smith2017federated} & 80.46 & 58.31  & 76.53  & 71.47 & 41.25 & 71.20 & 54.11 \\
PerFedAvg \citep{fallah2020personalized} & 82.27 & 67.20 & 67.36 & 72.05 &  52.49  & 68.45 & 71.51 \\
LG-Fed \citep{liang2020think} & 84.14  & 63.02 &  77.48 & 72.44 &38.76 & 70.37 & 62.08 \\
L2GD \citep{hanzely2020federated} & 81.04 & 59.98 & 71.96 & 72.13 & 42.84 & 70.67 & 66.18  \\
APFL  \citep{deng2020adaptive} & 83.77 & 72.29 & 82.39 & 78.20 & 55.44 & 69.87 & 70.74 \\
Ditto \citep{li2020ditto} & 85.39 & 70.34 & 80.36 & 78.91 &  $\mathbf{56.34}$ & 71.04 & 68.28 \\
FedPer \citep{arivazhagan2019federated} & 87.13 & 73.84 & 81.73 & 76.00 & 55.68 & 72.12 & 76.91 \\
\midrule
FedRep (Ours) & $\mathbf{87.70}$ & $\mathbf{75.68}$ & $\mathbf{83.27}$ & 79.15 & 56.10 &  $\mathbf{72.41}$ & $\mathbf{78.56}$ \\ 
\bottomrule
\end{tabular}
\end{center}
\vskip -0.1in
\end{table*}

\textbf{Implementation.}
In each experiment we sample a ratio $r\!=\!0.1$ of all the clients on every round. We initialize all models randomly and train for $T\!=\!100$ communication rounds for the CIFAR datasets, $T\!=\!50$ for Sent140, and $T\!=\!200$ for FEMNIST. In each case, for each local updates FedRep executes ten local epochs of SGD with momentum to train the local head, followed by one epoch for the representation in the case of CIFAR10 with $n\!=\!100$ and 5 epochs in all other cases. All other methods use the same number of local epochs as FedRep does for updating the representation. Accuracies are computed by taking the average local accuracies for all users over the final 10 rounds of communication, except for the fine-tuning methods. These accuracies are computed after locally training the head of the fully-trained global model for ten epochs for each client.

\begin{figure}[t]
\begin{center}
\centerline{\includegraphics[width=0.56\columnwidth]{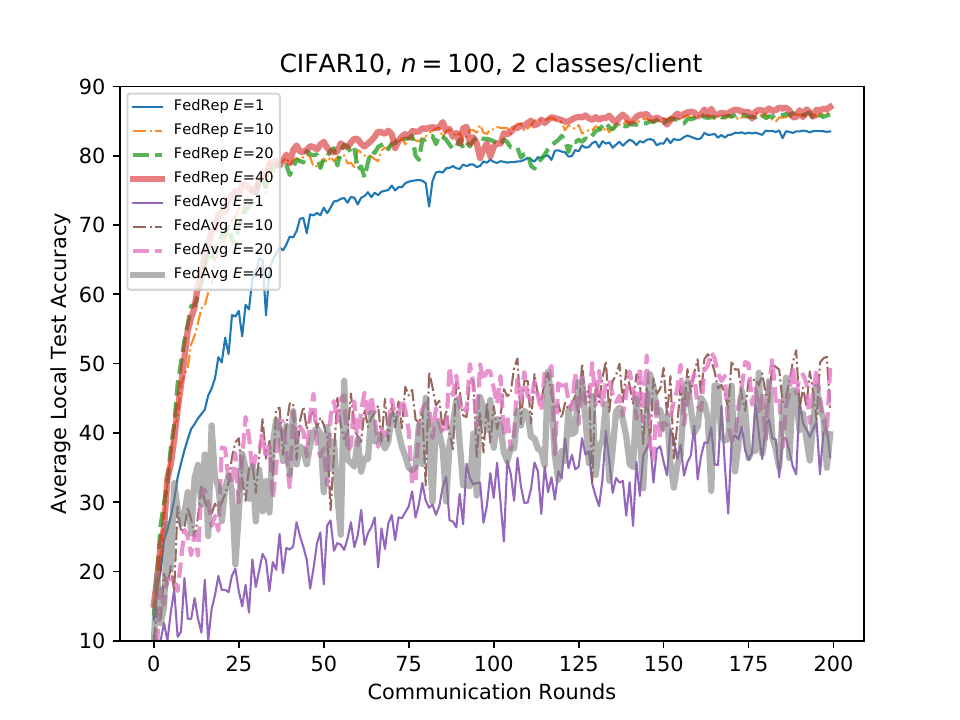}}
\caption{CIFAR10 local test errors for different numbers of local epochs $E$ for  FedAvg and for the heads in FedRep.}
\label{fig:cifar}
\end{center}
\vspace{-2mm}
\end{figure}

\textbf{Benefit of more local updates.} As mentioned in Section \ref{sec:intro}, a key advantage of our formulation is that it enables clients to run many local updates without causing divergence from the global optimal solution. We demonstrate an example of this in Figure \ref{fig:cifar}. Here, there are $n\!=\!100$ clients where each has $S\!=\!2$ classes of images. For FedAvg, we observe running more local updates does not necessarily improve the performance. In contrast, FedRep's performance is monotonically non-decreasing with the number of local epochs for the heads, i.e., {\em FedRep  is never hurt by more local computation on the heads.}  

\textbf{Robustness to varying levels of heterogeneity, number of clients and number of samples per client.} We show the average local test errors for all of the algorithms for a variety of settings in Table \ref{bigtable1}. Recall that for the CIFAR datasets, the number of training samples per client is equal to $50000/n$, so the columns with 100 clients have 500 training samples per client, and the column with 1000 clients has only 50 training samples per client. In all cases, FedRep is either the top-performing method or is very close to the top-performing method. Surprisingly, the fine-tuning methods perform very well, especially FedAvg+FT. The superior performance of FedAvg relative SCAFFOLD is likely because all settings involve partial client participation. 




\begin{figure}[t]
\begin{center}
\vspace{-2mm}
\centerline{\includegraphics[width=0.64\columnwidth]{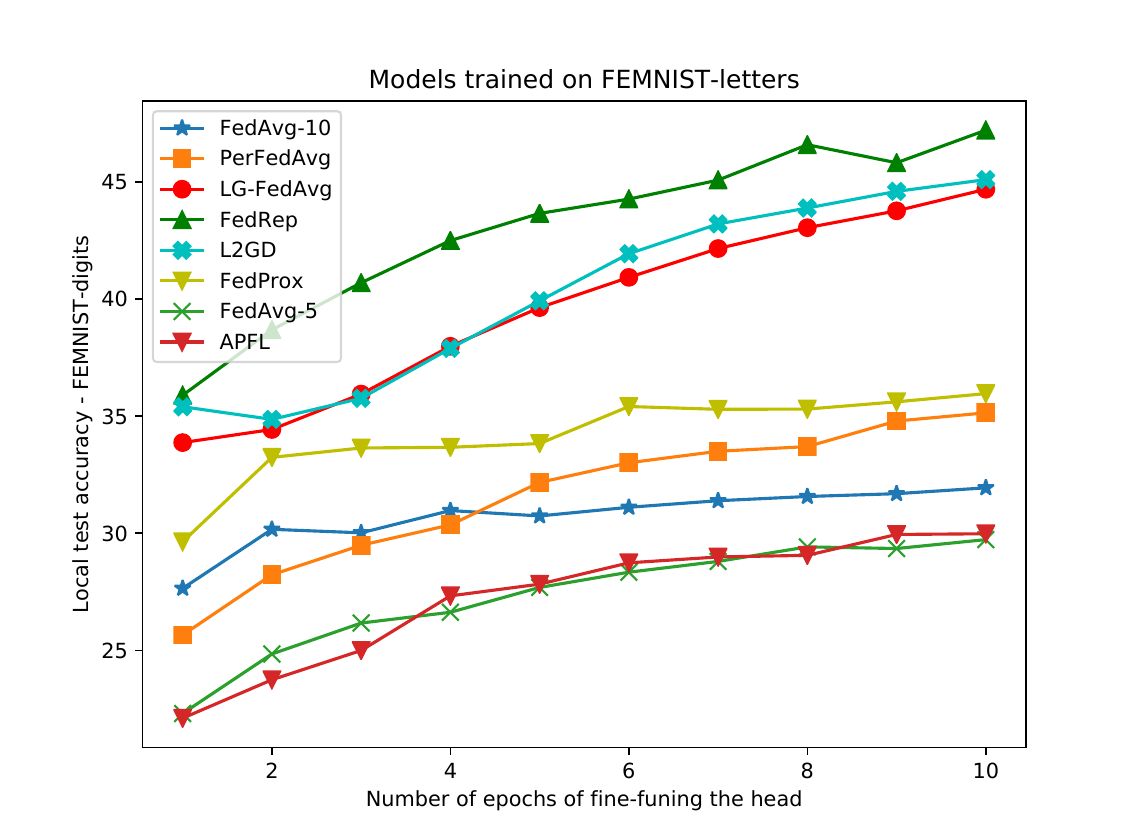}}
\vspace{-2mm}
\caption{Test accuracy on handwritten digits from FEMNIST after fine-tuning the head of models trained on FEMNIST-letters.}
\label{fig:new_clients}
\end{center}
\vspace{-6mm}
\end{figure}

\textbf{Generalization to new clients.} 
We also evaluate the strength of the representation learned by FedRep in terms of 
adaptation for new users. To do so, we first train FedRep, FedAvg, PerFedAvg, LG-FedAvg, APFL, L2GD and FedProx
in the usual setting on the partition of FEMNIST containing images of 10 handwritten letters (FEMNIST-letters). Then, we encounter clients with data from a different partition of the FEMNIST dataset, containing images of handwritten digits. We assume we have access to a dataset of 500 samples at this new client to fine tune the head. Using these, with each of the algorithms, we fine tune the head over multiple epochs while keeping the representation fixed. 
In Figure \ref{fig:new_clients}, we repeatedly sweep over the same 500 samples over multiple epochs to further refine the head, and plot the corresponding local test accuracy. As is apparent, FedRep has significantly better performance than these baselines.

\section{Discussion}

We introduce a novel representation learning framework and algorithm for federated learning, and we provide both theoretical and empirical justification for its utility in federated settings. In particular, our proposed framework exploits the structure of federating learning by (i) leveraging all clients' data to learn a global representation that enhances each client's model and can generalize to new users and (ii) leveraging the computational power of clients to run multiple local updates for learning their local heads. Our analysis further  shows that alternating minimization-descent efficiently learns linear representations, and is therefore relevant beyond federated learning. Future work remains to analyze the representation learning capabilities of FedRep in non-linear settings.

\section{Acknowledgements}
The research of Liam Collins is supported through ARO Grant W911NF-11-1-0265 and NSF Grant 2019844. The research of Sanjay Shakkottai is supported by ONR Grant N00014-19-1-2566 and NSF Grant 2019844. The research of Aryan Mokhtari is supported in part by NSF Grant 2007668, ARO Grant W911NF2110226, and the Machine Learning Laboratory at UT Austin. The research of Hamed Hassani is supported by NSF Grants 1837253, 1943064, 1934876, AFOSR Grant FA9550-20-1-0111, and DCIST-CRA.









\newpage

\appendix

\section{Additional Experimental Results}
\label{app:experiments}

\subsection{Synthetic Data: Further comparison with GD-GD}

\begin{figure}[H]
\vskip 0.2in
\begin{center}
\centerline{\includegraphics[width=1.1\columnwidth]{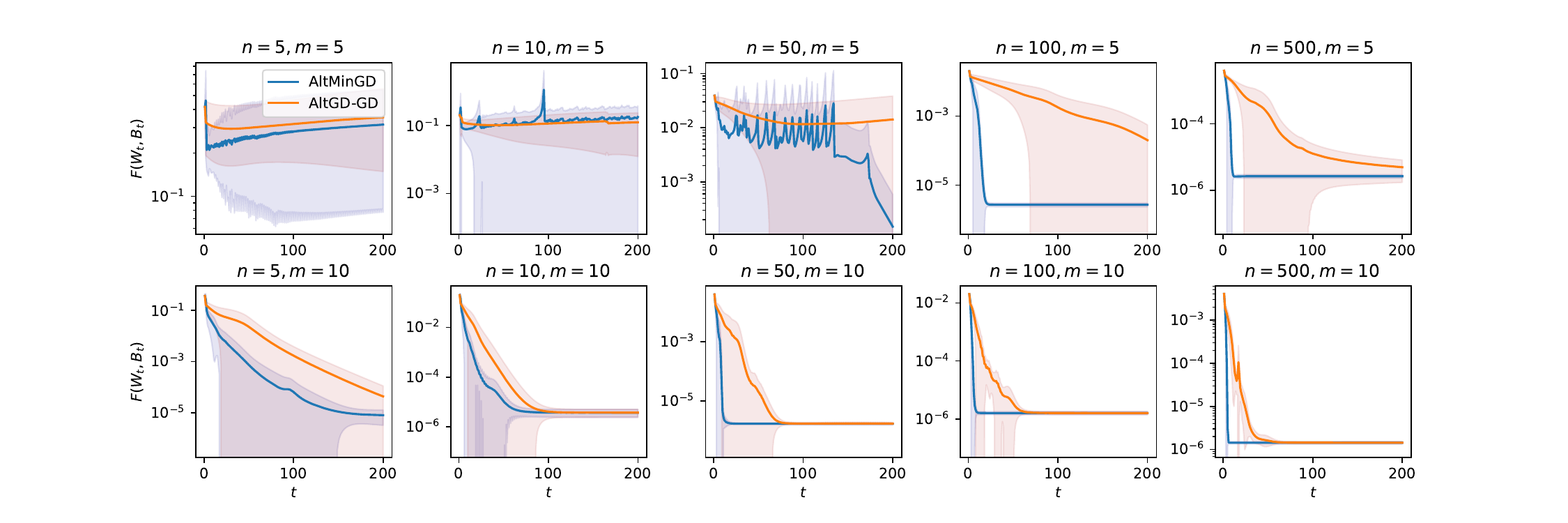}}
\caption{Function values for FedRep and GD-GD. The value of $m$ is fixed in each row and $n$ is fixed in each column. Here $r\!=\!1$ (full participation) and the average trajectories over 10 trials are plotted along with 95\% confidence intervals. Principal angle distances are not plotted as the results are very similar. {\em We see that the relative improvement of FedRep over GD-GD increases with $n$, highlighting the advantage of FedRep in settings with many clients.}}
\label{fig:synthetic1}
\end{center}
\vskip -0.2in
\end{figure}

\textbf{Further experimental details.} In the synthetic data experiments, the ground-truth matrices $\mathbf{W}^*$ and $\mathbf{B}^*$ were generated by first sampling each element as an i.i.d. standard normal variable, then taking the QR factorization of the resulting matrix, and scaling it by $\sqrt{k}$ in the case of $\mathbf{W}^\ast$. The clients each trained on the same $m$ samples throughout the entire training process. Test samples were generated identically as the training samples but without noise.
Both the iterates of FedRep and GD-GD were initialized with the SVD of the result of 10 rounds of projected gradient descent on the unfactorized matrix sensing objective as in Algorithm 1 in \cite{tu2016low}. We would like to note that FedRep exhibited the same convergence trajectories regardless of whether its iterates were initialized with random Gaussian samples or with the projected gradient descent procedure, whereas GD-GD was highly sensitive to its initialization, often not converging when initialized randomly.

\subsection{Real Data: Further experimental details}

\textbf{Datasets.} The CIFAR10 and CIFAR100 datasets \citep{krizhevsky2009learning} were generated by randomly splitting the training data into $Sn$ shards with $50,000/(Sn)$ images of a single class in each shard, as in \cite{mcmahan2017communication}. The full Federated-EMNIST (FEMNIST) dataset contains 62 classes of handwritten letters, but in Table \ref{bigtable1} we use a subset with only 10 classes of handwritten letters. In particular, we followed the same dataset generation procedure as in \cite{li2019feddane}, but used 150 clients instead of 200. When testing on new clients as in Figure \ref{fig:new_clients}, we use samples from 10 classes of handwritten digits from FEMNIST, i.e., the MNIST dataset. In this phase there are 100 new clients, each with 500 samples from 5 different classes for fine-tuning. The fine-tuned models are then evaluated on 100 testing samples from these same 5 classes. For Sent140, we randomly sample 183 clients (Twitter users) that each have at least 50 samples (tweets). Each tweet is either positive sentiment or negative sentiment. Statistics of both the FEMNIST and Sent140 datasets we use are given in Table \ref{table:data}. For both FEMNIST and Sent140 we use the LEAF framework \citep{caldas2018leaf}.

\textbf{Hyperparameters.} As in \cite{liang2020think}, all methods use SGD with momentum with parameter equal to 0.5.
In Table \ref{bigtable1}, for CIFAR10, CIFAR100, and FEMNIST the local sample batch size is 10 and for Sent140 it is 4. The participation rate $r$ is always 0.1, besides in the fine-tuning phases in Figure \ref{fig:new_clients}, in which all clients are sampled in each round. For each dataset learning rates were tuned in $\{0.001, 0.01, 0.1 \}$.
We observed that the optimal learning rates for FedAvg were also typically the optimal base learning rates for the other methods, so we used the same base learning rates for all methods for each dataset, which was $0.01$ in all cases, unless stated otherwise. Note that the batch size and learning rate for CIFAR10 used in Table \ref{bigtable1} differs from the standard setting of a batch size of 50 and learning rate of 0.1 \citep{mcmahan2017communication}, but we observed improved performance for all methods by using $(10,0.01)$ instead. In particular, the simulation in Figure \ref{fig:cifar}, the standard setting of $(50,0.1)$ is used, but the accuracies are worse than those reported in Table \ref{bigtable1} for both FedAvg and FedRep. Additionally, in Table \ref{bigtable1}, for CIFAR10 with $(n,S)=(100,2)$ and $(n,S)=(100,5)$, we executed $1$ local epoch of SGD with momentum for the representation for FedRep and 1 local epoch for all other methods. For all other datasets we executed 5 local epochs for the representation for FedRep and for the local updates for all other methods.

\textbf{Evaluation.} As mentioned in the main body, in Table \ref{bigtable1}, we initialize all methods randomly and train for $T=100$ communication rounds for the CIFAR datasets, $T=200$ for FEMNIST, and $T=50$ for Sent140. The accuracy shown is the average local test accuracy over all users over the final ten communication rounds, besides for the fine-tuning results, in which case we report the average local test accuracies of the locally fine-tuned models over all users, after the global model has been fully trained. We repeat the entire training and evaluation process five times for each model and dataset and report the averages in Table \ref{bigtable1}.


\textbf{Implementations.} Our code is available at \url{https://github.com/lgcollins/FedRep}.
We adapt the Pytorch codebase from \cite{liang2020think}, and used the implementations of FedAvg, Fed-MTL and LG-FedAvg from this repository. For consistency we use this same codebase to implement FedRep, FedPer, SCAFFOLD, FedProx, APFL, Ditto, L2GD, and Per-FedAvg. As in the experiments in \cite{liang2020think}, we used a 5-layer CNN with two convolutional layers for CIFAR10 and CIFAR100 followed by three fully-connected layers. For FEMNIST, we use an MLP with two hidden layers, and for Sent140 we use a pre-trained 300-dimensional GloVe embedding\footnote{Pennington,  J.,  Socher,  R.,  and Manning,  C. D.   Glove: Global vectors for word representation. In {\em Proceedings of the 2014 conference on empirical methods in natural language processing (EMNLP)}, pp. 1532–1543, 2014.} and train RNN with an LSTM module followed by two fully-connected decoding layers.

For FedRep, we treated the head as the weights and biases of the final fully-connected layer in each of the models. For LG-FedAvg, we treated the first two convolutional layers of the model for CIFAR10 and CIFAR100 as the local representation, and the fully-connected layers as the global parameters, and the input layer and hidden layers as the global parameters. For FEMNIST, we set all parameters besides those in the output layer as the local representation parameters. For Sent140, we set the RNN module to be the local representation and the decoder to be the global parameters. Unlike in the paper introducing LG-FedAvg \citep{liang2020think}, we did not initialize the models for all methods with the solution of many rounds of FedAvg (instead, we initialized randomly) and we computed the local test accuracy as the average local test accuracy over the final ten communication rounds, rather than the average of the maximum local test accuracy for each client over the entire training procedure.

For L2GD we executed multiple epochs of local SGD (discussed above) instead of one step of GD in the local update in order for reasonable comparison with the other methods. We also set $p=0.9$, thus the local parameters are trained on {10}\% of the communication rounds.
We tuned $\alpha$ in $\{0.05, 0.1, 0.25, 0.5,0.75\}$ and 
we tuned  $\lambda$ over $\{ 1, 0.5\}$. We used $(\alpha,\lambda) = (0.25,1)$ in all cases besides the $(n,S)=(100,5)$ case for CIFAR100, for which we used $\alpha=0.1$. Also, for FEMNIST we improved performance by using a learning  rate of 0.001 instead of 0.01. For APFL, we used a fixed $\alpha$ that we tuned in $\{0.1, 0.25,0.5,0.75\}$, and chose $\alpha=0.25$ for all cases besides the most heterogeneous CIFAR versions, namely $(n,S)=(100,2)$ for CIFAR10 and $(n,S)=(100,25)$ for CIFAR100.
For Ditto we tuned $\lambda$ among $\{0.25,0.5,0.75,1\}$, and used $\lambda=0.75$ for all cases besides CIFAR100, for which we used $\lambda=1$.
For PerFedAvg, we used an inner learning rate of $10^{-4}$ and 8 samples as the support set and 2 samples as the target set in each local meta-gradient update. We used the Hessian-free version.
For FedProx we tuned $\mu$ among $\{ 0.05, 0.1, 0.25, 0.5\}$, and used $\mu = 0.1$ for CIFAR and $\mu=0.25$ for FEMNIST and Sent140. For SCAFFOLD we used a global learning rate of 1 in all cases besides FEMNIST, for which  0.5 was superior. 







\begin{table}[H]
\caption{Dataset statistics.}
\vspace{1mm}
\label{table:data}
\begin{center}
\begin{small}
\begin{sc}
\begin{tabular}{lccc}
\toprule
Dataset & Number of users ($n$) & Avg samples/user & Min samples/user \\
\midrule
FEMNIST      & 150  & 148  & 50 \\
Sent140     & 183  &  72  & 50 \\
\bottomrule
\end{tabular}
\end{sc}
\end{small}
\end{center}
\end{table}

\section{Proof of Main Result} \label{app:proof_main}

\subsection{Preliminaries.}

We start by defining some notions used throughout the proof.
\begin{definition}
For a random vector $\mathbf{x} \in \mathbb{R}^d$ and a fixed matrix $\mathbf{A} \in \mathbb{R}^{d_1 \times d_2}$, the vector $\mathbf{A}^\top \mathbf{x}$  is called {\em $\|\mathbf{A}\|_2$-sub-gaussian} if $\mathbf{y}^\top \mathbf{A}^\top \mathbf{x}$ is sub-gaussian with sub-gaussian norm $\mathcal{O}(\|\mathbf{A}\|_2  \|\mathbf{y}\|_2)$ for all $\mathbf{y} \in \mathbb{R}^{d_2}$, i.e. $\mathbb{E}[\exp({\mathbf{y}^\top \mathbf{A}^\top \mathbf{x}})] \leq \exp\left({\|\mathbf{y}\|_2^2 \|\mathbf{A}\|_2^2/2}\right)$.
\end{definition}

\begin{definition} \label{def:inco}
A rank-$k$ matrix $
\mathbf{M} \in \mathbb{R}^{d_1 \times d_2}$ is {\em $\mu$-row-wise incoherent} if 
  $ \max_{i \in [d_1]} \|\mathbf{m}_i\|_2 \leq (\nicefrac{\mu \sqrt{d_2}}{\sqrt{d_1}}) \|\mathbf{M}\|_F$, 
where $\mathbf{m}_i \in \mathbb{R}^{d_2}$ is the $i$-th row of $\mathbf{M}$.
\end{definition}

Note that Assumption \ref{assump:norm} implies that $\mathbf{W}^\ast$ is row-wise incoherent with parameter 1.

We use hats to denote orthonormal matrices (a matrix is called orthonormal if its set of columns is an orthonormal set). By Assumption \ref{assump:norm}, the ground truth representation $\mathbf{B}^\ast$ is orthonormal, so from now on we will write it as $\mathbf{\hat{B}}^\ast$. Likewise, we will denote the iterates $\mathbf{{B}}^t$ as $\mathbf{\hat{B}}^t$.

For a matrix $\mathbf{W} \in \mathbb{R}^{n \times k}$ and a random set of indices $\mathcal{I} \in [n]$ of cardinality $rn$, define $\mathbf{W}_\mathcal{I} \in \mathbb{R}^{rn \times k}$ as the matrix formed by taking the rows of $\mathbf{W}$ indexed by $\mathcal{I}$. Define $\bar{\sigma}_{\max, \ast} \coloneqq \max_{\mathcal{I} \in [n], |\mathcal{I}| = rn} \sigma_{\max} (\frac{1}{\sqrt{rn}}\mathbf{W}^\ast_{\mathcal{I}})$ and $\bar{\sigma}_{\min, \ast} \coloneqq \min_{\mathcal{I} \in [n], |\mathcal{I}| = rn} \sigma_{\min}(\frac{1}{\sqrt{rn}} \mathbf{W}^\ast_{\mathcal{I}})$, i.e. the maximum and minimum singular values of any matrix that can be obtained by taking $rn$ rows of $\frac{1}{\sqrt{rn}}\mathbf{W}^\ast$. Note that by Assumption \ref{assump:norm}, each row of $\mathbf{W}^\ast$ has norm $\sqrt{k}$, so $\frac{1}{\sqrt{rn}}$ acts as a normalizing factor such that $\|\frac{1}{\sqrt{rn}}\mathbf{W}^\ast_{\mathcal{I}}\|_F = \sqrt{k}$.  In addition, define $\kappa = \bar{\sigma}_{\max, \ast}/\bar{\sigma}_{\min, \ast}$.

Let $i$ now be an index over $[rn]$, and let $i'$ be an index over $[n]$.
For random batches of samples $\{ \{ (\mathbf{x}_{i}^{j},y_i^j ) \}_{j=1}^m \}_{i=1}^{rn}$,
define the random linear operator $\mathcal{A}: \mathbb{R}^{rn \times d} \rightarrow \mathbb{R}^{rnm}$ as 
$
    \mathcal{A}(\mathbf{M}) = [\langle \mathbf{A}_{i,j}, \mathbf{M} \rangle]_{1\leq i \leq rn, 1\leq j \leq m} \in \mathbb{R}^{rnm}.
$
Here, $\mathbf{A}_{i,j} \coloneqq \mathbf{e}_i  ({\mathbf{x}_i^{j}})^\top $, where $\mathbf{e}_i$ is the $i$-th standard vector in $\mathbb{R}^{rn}$, and $\mathbf{M} \in \mathbb{R}^{rn \times d}$. 
Then, the loss function in \eqref{obj:glob_linear} is equivalent to 
\begin{align}
    \min_{\mathbf{B} \in \mathbb{R}^{d \times k}, \mathbf{W} \in \mathbb{R}^{n \times k}} \{ &
    F(\mathbf{B},\mathbf{W}) \coloneqq \frac{1}{2rnm} \mathbb{E}_{\mathcal{A}, \mathcal{I}} \left[\| \mathbf{Y} - \mathcal{A}(\mathbf{W}_\mathcal{I}\mathbf{B}^\top)\|_2^2 \right] \}, \label{linop_app} 
\end{align}
where $\mathbf{Y}  =\mathcal{A}(\mathbf{W}^\ast_\mathcal{I} \mathbf{\hat{B}}^{\ast^\top}) \in \mathbb{R}^{rnm}$ is a concatenated vector of labels. It is now easily seen that the problem of recovering $\mathbf{W}^\ast \mathbf{\hat{B}}^{\ast^\top}$ from finitely-many measurements $\mathcal{A}(\mathbf{W}^\ast_\mathcal{I} \mathbf{\hat{B}}^{\ast^\top})$ is an instance of matrix sensing. 
Moreover, the updates of FedRep satisfy the following recursion:
\begin{align}
    {\mathbf{W}}^{t+1}_{\mathcal{I}^t} &= 
    \argmin_{{\mathbf{W}}_{\mathcal{I}^t} \in \mathbb{R}^{rn\times k}} \frac{1}{2rnm}\| \mathcal{A}^t({\mathbf{W}}_{\mathcal{I}^t}^\ast\mathbf{\hat{B}}^{\ast^\top} -  {\mathbf{W}}_{\mathcal{I}^t}\mathbf{\hat{B}}^{t^\top})\|_2^2 \label{up_ww} \\
  \mathbf{\bar{B}}^{t+1} &= \mathbf{\hat{B}}^t - \!\frac{\eta}{rnm} 
  \left((\mathcal{A}^t)^\dagger \mathcal{A}^t({\mathbf{W}}^{t+1}_{\mathcal{I}^t}\mathbf{\hat{B}}^{t^\top}- {\mathbf{W}}_{\mathcal{I}^t}^\ast\mathbf{\hat{B}}^{\ast^\top})\right)^\top {\mathbf{W}}^{t+1}_{\mathcal{I}^t} \label{up_b} \\
  \mathbf{\hat{B}}^{t+1}, \mathbf{{R}}^{t+1} &= \text{QR}(\mathbf{\bar{B}}^{t})
\end{align}
where $\mathcal{A}^t$ is an instance of $\mathcal{A}$,  $(\mathcal{A}^t)^\dagger$ is the adjoint operator of $\mathcal{A}^t$, i.e. $(\mathcal{A}^t)^\dagger (\mathbf{M}) = \sum_{i=1}^{rn}\sum_{j=1}^m (\langle \mathbf{A}^{t,j}_{i},  {\mathbf{M}} \rangle) \mathbf{A}_{i}^{{t,j}}$, and QR$(\cdot)$ is the QR factorization.
Note that for the purposes of analysis, it does not matter how $\mathbf{w}_{i'}^{t+1}$ is computed for all $i' \notin \mathcal{I}^t$, as these vectors do not affect the computation of $\mathbf{B}^{t+1}$. Moreover, our analysis does not rely on any particular properties of the batches $\mathcal{I}^1,\dots, \mathcal{I}^T$ other than the fact that they have cardinality $rn$, so without loss of generality we assume $\mathcal{I}^t = [rn]$ for all $t=1,...T$ and drop the subscripts $\mathcal{I}^t$ on $\mathbf{W}^t$. Further, since our analysis focuses on a particular iteration $t$, we will drop the superscript $t$ on $\mathcal{A}^t$ and each $\mathbf{A}_{i}^{t,j}$ and $(\mathbf{x}_i^{t,j}, y_i^{t,j})$ for ease of notation (while noting that each iteration requires a new batch of i.i.d. data).



\subsection{Auxilliary Lemmas}

We start by computing the update for $\mathbf{W}$.

\begin{lemma} \label{lem:defu}
In the linear version of FedRep, update for $\mathbf{W}$ is:
\begin{align}
    \mathbf{W}^{t+1} = \mathbf{{W}}^\ast  \mathbf{\hat{B}}^{\ast^\top} \mathbf{\hat{B}}^t - \mathbf{F}
\end{align}
where $\mathbf{F}$ is defined in equation \eqref{deff} below.
\end{lemma}

\begin{proof}
We adapt the argument from Lemma 4.5 in \citep{Jain_2013} to compute the update for $\mathbf{W}^{t+1}$, and borrow heavily from their notation. 

Let $\mathbf{w}_p^{t+1}$ (respectively $\mathbf{\hat{b}}_p^{t+1}$) be the $p$-th column of $\mathbf{W}^t$ (respectively $\mathbf{\hat{B}}^t$).
Since $\mathbf{W}^{t+1}$ minimizes $\tilde{F}(\mathbf{W}, \mathbf{\hat{B}}^{t}) \coloneqq \frac{1}{2rnm}\| \mathcal{A}({\mathbf{W}}^\ast(\mathbf{\hat{B}}^{\ast})^\top -  {\mathbf{W}}(\mathbf{B}^{t})^\top)\|_2^2$ with respect to $\mathbf{W}$, we have $\nabla_{\mathbf{w}_p} \tilde{F}(\mathbf{W}^{t+1}, \mathbf{\hat{B}}^t)=\mathbf{0}$ for all $p \in [k]$. Thus, for any $p \in[k]$, we have
\begin{align*}
 \mathbf{0} &= \nabla_{\mathbf{w}_p} \tilde{F}(\mathbf{W}^{t+1},\mathbf{\hat{B}}^t) \\
 &= \frac{1}{rnm}\sum_{i=1}^{rn} \sum_{j=1}^m\left(\langle  \mathbf{A}_{i,j}, \mathbf{W}^{t+1} {(\mathbf{\hat{B}}^{t}})^\top - \mathbf{W}^\ast {(\mathbf{\hat{B}}^\ast})^\top \rangle \right)
	\mathbf{A}_{i,j} \mathbf{\hat{b}}^{t}_p  \\
&= \frac{1}{rnm}\sum_{i=1}^{rn} \sum_{j=1}^m\left(\sum_{q=1}^k {(\mathbf{\hat{b}}_q^{t}})^\top  \mathbf{A}_{i,j}^\top \mathbf{w}_q^{t+1} - \sum_{q=1}^k ({\mathbf{\hat{b}}_q^*})^\top  \mathbf{A}_{i,j}^\top \mathbf{{w}}_q^* \right)
	\mathbf{A}_{i,j} \mathbf{\hat{b}}^{t}_p  
\end{align*}
This implies
\begin{align}
\frac{1}{m}\sum_{q=1}^k \left( \sum_{i=1}^{rn} \sum_{j=1}^m  \mathbf{A}_{i,j} \mathbf{\hat{{b}}}^{t}_p ({\mathbf{\hat{b}}_q^{t}})^\top  \mathbf{A}_{i,j}^\top\right)  \mathbf{{w}}_q^{t+1} 
	&= \frac{1}{m}\sum_{q=1}^k\left(\sum_{i=1}^{rn} \sum_{j=1}^m \mathbf{A}_{i,j} \mathbf{\hat{{b}}}^{t}_p ({\mathbf{\hat{b}}_q^*})^\top  \mathbf{A}_{i,j}^\top \right) \mathbf{{w}}_q^* \label{implies}
\end{align}
To solve for $\mathbf{w}^{t+1}$, we define $\mathbf{G}$, $\mathbf{C}$, and $\mathbf{D}$ as $rnk$-by-$rnk$ block matrices, as follows: 
{\small 
\begin{align}
& \mathbf{G} \coloneqq \left[\begin{array}{ccc}
                 \mathbf{G}_{11} & \cdots & \mathbf{G}_{1k} \\
		 \vdots  	& \ddots &\vdots \\
                 \mathbf{G}_{k1} & \cdots & \mathbf{G}_{kk} \\
                \end{array} \right]
\mbox{ , }
 \mathbf{C} \coloneqq \left[\begin{array}{ccc}
                 \mathbf{C}_{11} & \cdots & \mathbf{C}_{1k} \\
		 \vdots  	& \ddots &\vdots \\
                 \mathbf{C}_{k1} & \cdots & \mathbf{C}_{kk} \\
                \end{array} \right]\mbox{ , }\mathbf{D}\coloneqq \left[\begin{array}{ccc}
                 \mathbf{D}_{11}  & \cdots & \mathbf{D}_{1k} \\
		 \vdots 	& \ddots &\vdots \\
                 \mathbf{D}_{k1} & \cdots & \mathbf{D}_{kk}\end{array} \right]\label{eq:BCD_k}
\end{align}}
where, for $p,q\in [k]$: $\mathbf{G}_{pq} \coloneqq \frac{1}{m}\sum_{i=1}^{rn} \sum_{j=1}^m \mathbf{A}_{i,j} \mathbf{\hat{b}}_p^{t}{\mathbf{\hat{b}}_q^{t^\top}}  \mathbf{A}_{i,j}^\top \in \mathbb{R}^{rn \times rn} $,  $\mathbf{C}_{pq} \coloneqq \frac{1}{m}\sum_{i=1}^{rn} \sum_{j=1}^m \mathbf{A}_{i,j} \mathbf{\hat{b}}_p^{t}{(\mathbf{\hat{b}}_q^*})^\top  \mathbf{A}_{i,j}^\top \in \mathbb{R}^{rn \times rn} ,$ and, $\mathbf{D}_{pq} \coloneqq \langle {\mathbf{\hat{b}}_p^{t}, \mathbf{\hat{b}}_q^*}\rangle\mathbf{I}_{rn} \in \mathbb{R}^{rn \times rn}.$ Recall that $ \mathbf{\hat{b}}_p^{t}$ is the $p$-th column of $\mathbf{\hat{B}}^{t}$ and $\mathbf{\hat{b}}_q^*$ is the $q$-th column of $\mathbf{\hat{B}}^*$. 
Further, define  $$  \widetilde{\mathbf{w}}^{t+1}=\left[\begin{matrix}\mathbf{{w}}_1^{t+1}\\\vdots\\\mathbf{w}_k^{t+1}\end{matrix}\right]\in \mathbb{R}^{rnk},\quad  \widetilde{\mathbf{w}}^{\ast}=\left[\begin{matrix}\mathbf{{w}}_1^{\ast}\\\vdots\\\mathbf{{w}}_k^{\ast}\end{matrix}\right]\in \mathbb{R}^{rnk}.$$
Then, by \eqref{implies}, we have
\begin{align*}
   \widetilde{\mathbf{w}}^{t+1}&=  \mathbf{G}^{-1}\mathbf{C}\widetilde{\mathbf{w}}^* \\
 	&= \mathbf{D}\widetilde{\mathbf{w}}^\ast - \mathbf{G}^{-1}\left(\mathbf{GD}-\mathbf{C}\right)\widetilde{\mathbf{w}}^*
\end{align*}
where we can invert $\mathbf{G}$ conditioned on the event that its minimum singular value is strictly positive, which Lemma
\ref{lem:ginv} shows holds with high probability. Now consider the $p$-th block of $\widetilde{\mathbf{w}}^{t+1}$, and let $ (\left(\mathbf{GD}-\mathbf{C}\right)\mathbf{w}^*)_p$ denote the $p$-th block of $\left(\mathbf{GD}-\mathbf{C}\right)\mathbf{w}^*$. We have 
\begin{align}
    \widetilde{\mathbf{w}}^{t+1}_p &= \sum_{q=1}^k \langle {\mathbf{\hat{b}}_p^{t}, \mathbf{\hat{b}}_q^*}\rangle \mathbf{w}^\ast_q - (\mathbf{G}^{-1}\left(\mathbf{GD}-\mathbf{C}\right)\mathbf{w}^*)_p \nonumber \\
    &=  \left(\sum_{q=1}^k \mathbf{w}^\ast_q (\mathbf{\hat{b}}_p^{*})^\top\right) \mathbf{\hat{b}}_q^t  - (\mathbf{G}^{-1}\left(\mathbf{GD}-\mathbf{C}\right)\mathbf{w}^*)_p \nonumber \\
    &= \left(\mathbf{W}^\ast (\mathbf{\hat{B}}^{*})^\top\right) \mathbf{\hat{b}}_q^t  - (\mathbf{G}^{-1}\left(\mathbf{GD}-\mathbf{C}\right)\mathbf{w}^*)_p
\end{align}
By constructing $\mathbf{W}^{t+1}$ such that the $p$-th column of $\mathbf{W}^{t+1}$ is $\mathbf{w}^{t+1}_p$ for all $p \in [k]$, we obtain
\begin{align}
    \mathbf{W}^{t+1} &= \mathbf{{W}}^\ast  \mathbf{\hat{B}}^{\ast}(\mathbf{\hat{B}}^{t})^\top - \mathbf{F}
\end{align}
where 
\begin{align}
    \mathbf{F} = [(\mathbf{G}^{-1}(\mathbf{GD} - \mathbf{C}) \mathbf{\widetilde{w}}^\ast)_1, \dots,  (\mathbf{G}^{-1}(\mathbf{GD} - \mathbf{C}) \mathbf{\widetilde{w}}^\ast)_k ] \label{deff}
\end{align}
and $(\mathbf{G}^{-1}(\mathbf{GD} - \mathbf{C}) \mathbf{\widetilde{w}}^\ast)_p$ is the $p$-th $n$-dimensional block of the $rnk$-dimensional vector $\mathbf{G}^{-1}(\mathbf{GD} - \mathbf{C}) \mathbf{\widetilde{w}}^\ast$.
\end{proof}

Next we bound the Frobenius norm of the matrix $\mathbf{F}$, which requires multiple steps. First, we establish some helpful notations. We drop superscripts indicating the iteration number $t$ for simplicity.


Again let ${\mathbf{w}}^\ast$ be the $rnk$-dimensional vector formed by stacking the columns of ${\mathbf{W}}^\ast$, and let $\mathbf{\hat{b}}_p$ (respectively $\mathbf{\hat{b}}_{q}^\ast$) be the $p$-th column of $\mathbf{\hat{B}}$ (respectively the $q$-th column of $\mathbf{\hat{B}}_\ast$).
Recall that $\mathbf{F}$ can be obtained by stacking $\mathbf{G}^{-1}(\mathbf{GD} - \mathbf{C}) {\mathbf{w}}^\ast$ into $k$ columns of length $n$, i.e. $\text{vec}(\mathbf{F}) =  \mathbf{G}^{-1}(\mathbf{GD} - \mathbf{C}) {\mathbf{w}}^\ast$. 
Further, $\mathbf{G} \in \mathbb{R}^{rnk \times rnk}$ is a block matrix whose blocks $\mathbf{G}_{pq} \in \mathbf{R}^{rn \times rn}$ for $p,q \in [k]$ are given by:
\begin{align}
    \mathbf{G}_{pq} &= \frac{1}{m}\sum_{i = 1}^{rn} \sum_{j=1}^m \mathbf{A}_{i,j} \mathbf{\hat{b}}_p \mathbf{\hat{b}}_q^\top \mathbf{A}_{i,j}^\top \nonumber \\
    &= \frac{1}{m} \sum_{i = 1}^{rn} \sum_{j=1}^m \mathbf{e}_{i} (\mathbf{x}_{i}^j)^\top \mathbf{\hat{b}}_p \mathbf{\hat{b}}_q^\top \mathbf{x}_{i}^j  \mathbf{e}_{i}^\top 
\end{align}
So, each $\mathbf{G}_{pq}$ is diagonal with diagonal entries
\begin{align}
    (\mathbf{G}_{pq})_{ii} &= \frac{1}{m}\sum_{j=1}^m  (\mathbf{x}_{i}^j)^\top \mathbf{\hat{b}}_p \mathbf{\hat{b}}_q^\top \mathbf{x}_{i}^j =   \mathbf{\hat{b}}_p^\top \Bigg(\frac{1}{m}\sum_{j=1}^m  \mathbf{x}_{i}^j (\mathbf{x}_{i}^j)^\top\Bigg) \mathbf{\hat{b}}_q
\end{align}
Define $\mathbf{\Pi}^i \coloneqq \frac{1}{m}\sum_{j=1}^m  \mathbf{x}_{i}^j (\mathbf{x}_{i}^j)^\top$ for all $i \in [rn]$.
Similarly as above, each block $\mathbf{C}_{pq}$ of $\mathbf{C}$ is diagonal with entries
\begin{align}
    (\mathbf{C}_{pq})_{ii} &= \mathbf{\hat{b}}_p^\top \mathbf{\Pi}^i \mathbf{\hat{b}}_{\ast,q}
\end{align}
Analogously to the matrix completion analysis in \citep{Jain_2013}, we define the following matrices, for all $i \in [rn]$:
\begin{align}
    \mathbf{G}^i \coloneqq \left[ \mathbf{\hat{b}}_p^\top \mathbf{\Pi}^i \mathbf{\hat{b}}_q \right]_{1\leq p,q \leq k} = \mathbf{\hat{B}}^\top \mathbf{\Pi}^i \mathbf{\hat{B}}, \quad \mathbf{C}^i \coloneqq \left[ \mathbf{\hat{b}}_p^\top \mathbf{\Pi}^i \mathbf{\hat{b}}_{\ast,q} \right]_{1\leq p,q \leq k} = \mathbf{\hat{B}}^\top \mathbf{\Pi}^i \mathbf{\hat{B}}_\ast
\end{align}
In words, $\mathbf{G}^i$ is the $k \times k$ matrix formed by taking the $i$-th diagonal entry of each block $\mathbf{G}_{pq}$, and likewise for $\mathbf{C}^i$. Recall that $\mathbf{D}$ also has diagonal blocks, in particular $\mathbf{D}_{pq} = \langle \mathbf{\hat{B}}_p, \mathbf{\hat{B}}_{{q}}^\ast \rangle \mathbf{I}_d$, thus we also define
$\mathbf{D}^i \coloneqq [\langle \mathbf{\hat{B}}_p, \mathbf{\hat{B}}_{{q}}^\ast \rangle]_{1\leq p,q \leq k} = \mathbf{\hat{B}}^\top \mathbf{\hat{B}}_\ast$. 

Using this notation we can decouple $\mathbf{G}^{-1}(\mathbf{GD} - \mathbf{C}) {\mathbf{w}}^\ast$ into $i$ subvectors. Namely, let ${\mathbf{w}}^{\ast}_i \in \mathbb{R}^k$ be the vector formed by taking the $((p-1)rn +i)$-th elements of ${\mathbf{w}}^\ast$ for $p=0,...,k-1$, and similarly, let $\mathbf{f}_i$ be the vector formed by taking the $((p-1)rn +i)$-th elements of $\mathbf{G}^{-1}(\mathbf{GD} - \mathbf{C}) {\mathbf{w}}^\ast$ for $p=0,...,k-1$. Then
\begin{align} \label{fi}
    \mathbf{f}_i = (\mathbf{G}^i)^{-1} (\mathbf{G}^i \mathbf{D}^i - \mathbf{C}^i){\mathbf{w}}^{\ast}_i
\end{align}
is the $i$-th row of $\mathbf{F}$. 
Now we control $\|\mathbf{F}\|_F$.


\begin{lemma} \label{lem:ginv}
Let $\delta_k = c\frac{k^{3/2} \sqrt{\log(rn)}}{\sqrt{m}}$ for some absolute constant $c$, then
\begin{align*}
    \|\mathbf{G}^{-1}\|_2 \leq \frac{1}{1-\delta_k} 
\end{align*}
with probability at least $1 - e^{-111 k^3 \log(rn)}$.
\end{lemma}
\begin{proof}
We must lower bound $\sigma_{\min}(\mathbf{G})$. For some vector $\mathbf{z} \in \mathbb{R}^{rnk}$, let $\mathbf{z}^i \in \mathbb{R}^k$ denote the vector formed by taking the $((p-1)rn +i)$-th elements of ${\mathbf{z}}$ for $p=0,...,k-1$. Since $\mathbf{G}$ is symmetric, we have
\begin{align}
    \sigma_{\min}(\mathbf{G}) &= \min_{\mathbf{z}:\|\mathbf{z}\|_2 =1} \mathbf{z}^\top \mathbf{G} \mathbf{z} \nonumber \\
    &= \min_{\mathbf{z}:\|\mathbf{z}\|_2 =1} 
    \sum_{i=1}^{rn} (\mathbf{z}^i)^\top \mathbf{G}^i \mathbf{z}^i \nonumber \\
    &=  \min_{\mathbf{z}:\|\mathbf{z}\|_2 =1} 
    \sum_{i=1}^{rn} (\mathbf{z}^i)^\top \mathbf{\hat{B}}^\top \mathbf{\Pi}^i \mathbf{\hat{B}} \mathbf{z}^i \nonumber \\
    &\geq \min_{i \in [rn]} 
    \sigma_{\min}(\mathbf{\hat{B}}^\top \mathbf{\Pi}^i \mathbf{\hat{B}} ) \nonumber
\end{align}
Note that the matrix $\mathbf{\hat{B}}^\top \mathbf{\Pi}^i \mathbf{\hat{B}} $ can be written as follows:
\begin{align}
    \mathbf{\hat{B}}^\top \mathbf{\Pi}^i \mathbf{\hat{B}} = \sum_{j=1}^m  \frac{1}{\sqrt{m}}\mathbf{\hat{B}}^\top \mathbf{x}_{i}^j \left(\frac{1}{\sqrt{m}}\mathbf{\hat{B}}^\top \mathbf{x}_{i}^j\right)^\top 
\end{align}
Let $\mathbf{v}_{i}^j \coloneqq \frac{1}{\sqrt{m}}\mathbf{\hat{B}}^\top \mathbf{x}_{i}^j$ for all $i\in [rn]$ and $j \in [m]$, and note that each $\mathbf{v}_{i}^j$ is i.i.d. $\frac{1}{\sqrt{m}}\mathbf{\hat{B}}$-sub-gaussian. 
Thus using the one-sided version of equation (4.22) (Theorem 4.6.1) in \citep{vershynin2018high}, we have
\begin{align}
    \sigma_{min}(\mathbf{\hat{B}}^\top \mathbf{\Pi}^i \mathbf{\hat{B}}) \geq 1 - C\left(\sqrt{\frac{{k}}{m}} + \frac{{z}}{\sqrt{m}} \right)
\end{align}
with probability at least $1-e^{-z^2}$ for $m \geq k$, $z \geq 0$ and some absolute constant $C$. Now let $\delta_k = C\left(\sqrt{\frac{{k}}{m}} + \frac{{z}}{\sqrt{m}} \right)$ to obtain
\begin{align}
    \sigma_{min}(\mathbf{\hat{B}}^\top \mathbf{\Pi}^i \mathbf{\hat{B}}) \geq 1 - \delta_k \label{omd}
\end{align}
with probability at least $1-e^{-(\delta_k\sqrt{m}/C - \sqrt{k})^2}$ for $m > k$. Now, choose $z$ such that $\delta_k= \frac{12C k^{3/2}\sqrt{\log(rn)}}{\sqrt{m}}$, we have that \eqref{omd} holds with probability at least 
\begin{align}
    1 -  \exp \left(- \left(12 {k^{3/2} \sqrt{\log(rn)} - \sqrt{k}} \right)^2 \right) &\geq 1 - \exp \left(-k( 12\sqrt{k} \sqrt{\log(rn)} - 1)^2 \right) \nonumber \\
    &\geq 1 - \exp \left(121 k^3  \log(rn)\right) \label{omd_prob}
\end{align}
Finally, taking a union bound over $i \in [n]$ yields
$
    \sigma_{\min}(\mathbf{G}) \geq 1 - \delta_k
$
with probability at least 
\begin{equation}
    1 - rn  \exp \left(-121 k^3  \log(rn)\right) \geq 1 - e^{-110 k^3 \log(rn)},
\end{equation}
completing the proof.
\end{proof}


\begin{lemma} \label{lem:gdif}
Let $\delta_k = c\frac{k^{3/2} \sqrt{\log(rn)}  }{\sqrt{m}}$ for some absolute constant $c$, then
\begin{align*}
    \|(\mathbf{GD} - \mathbf{C}) \mathbf{w}^\ast\|_2 \leq \delta_k  \|\mathbf{W}^\ast\|_2 \; \dist(\mathbf{\hat{B}}^t, \mathbf{\hat{B}}^\ast)
\end{align*}
with probability at least $1 - e^{-111k^2 \log(rn)}$.
\end{lemma}
\begin{proof}
For ease of notation we drop superscripts $t$. We define $\mathbf{H} = \mathbf{GD} - \mathbf{C}$ and 
\begin{align}
    \mathbf{H}^i &\coloneqq \mathbf{G}^i \mathbf{D}^i - \mathbf{C}^i = \mathbf{\hat{B}}^\top \mathbf{\Pi}^i \mathbf{\hat{B}} \mathbf{\hat{B}}^\top \mathbf{\hat{B}}^\ast- \mathbf{\hat{B}}^\top \mathbf{\Pi}^i \mathbf{\hat{B}}^\ast = \mathbf{\hat{B}}^\top \left(\frac{1}{m} \mathbf{X}_i^\top \mathbf{X}_i \right) (\mathbf{\hat{B}} \mathbf{\hat{B}}^\top -  \mathbf{I}_d)\mathbf{\hat{B}}^\ast, 
\end{align}
for all $i \in [rn]$.
Then we have
\begin{align}
    \|(\mathbf{GD} - \mathbf{C}) \mathbf{w}_\ast\|_2^2
    &= \sum_{i=1}^{rn} \| \mathbf{H}^i  \mathbf{w}_\ast^i\|_2^2 \nonumber \\
    &\leq \sum_{i=1}^{rn} \| \mathbf{H}^i \|_2^2 \|\mathbf{w}^\ast_i\|_2^2\nonumber \\
    &\leq \frac{k}{rn} \|\mathbf{W}^\ast\|_2^2  \sum_{i=1}^{rn}  \| \mathbf{H}^i \|_2^2  \label{rw1}
\end{align}
where the last inequality follows almost surely from Assumption \ref{assump:norm} (the $1$-row-wise incoherence of $\mathbf{{W}}^\ast$), the fact that $krn = \|\mathbf{W}^\ast\|_F^2 \leq k\|\mathbf{W}^\ast\|_2^2$ by Assumption \ref{assump:norm}, and the fact that $\mathbf{W}^\ast$ has rank $k$. It remains to bound $\frac{1}{rn}\sum_{i=1}^{rn} \|\mathbf{H}^i\|_2^2$. Although $\|\mathbf{H}^i\|_2$ is sub-exponential (as we will show), $\|\mathbf{H}^i\|_2^2$ is not sub-exponential, so we cannot directly apply standard concentration results. Instead, we compute a tail bound for each $\|\mathbf{H}^i\|_2^2$ individually, then 
then union bound over $i \in [rn]$. 
Let $\mathbf{U} \coloneqq \frac{1}{\sqrt{m}}\mathbf{X}_i(\mathbf{\hat{B}} \mathbf{\hat{B}}^\top -  \mathbf{I}_d)\mathbf{\hat{B}}^\ast$, then the $j$-th row of $\mathbf{U}$ is given by $$\mathbf{u}_j = \frac{1}{\sqrt{m}} \mathbf{\hat{B}}^{\ast^\top}(\mathbf{\hat{B}}\mathbf{\hat{B}}^\top - \mathbf{I}_d)\mathbf{x}_i^j,$$ and is $\frac{1}{\sqrt{m}} \mathbf{\hat{B}}^{\ast^\top}(\mathbf{\hat{B}}\mathbf{\hat{B}}^\top - \mathbf{I}_d)$-sub-gaussian. Likewise, define $\mathbf{V} \coloneqq \frac{1}{\sqrt{m}} \mathbf{X}_i \mathbf{\hat{B}}$, then the $j$-th row of $\mathbf{V}$ is $$ \mathbf{v}_j = \frac{1}{\sqrt{m}} \mathbf{\hat{B}}^\top\mathbf{x}_i^j,$$ therefore is $\frac{1}{\sqrt{m}} \mathbf{\hat{B}}$-sub-gaussian. 
We leverage the sub-gaussianity of the rows of $\mathbf{U}$ and $\mathbf{V}$ to make a similar concentration argument as in Proposition 4.4.5 in \cite{vershynin2018high}. First, let $\mathcal{S}^{k-1}$ denote the unit sphere in $k$ dimensions, and let $\mathcal{N}_k$ be a $\frac{1}{4}$-th net of cardinality $|\mathcal{N}_k|\leq 9^k$, which exists by Corollary 4.2.13 in \cite{vershynin2018high}. Next, using equation 4.13 in \cite{vershynin2018high}, we obtain
\begin{align}
  \|(\mathbf{\hat{B}}^\ast)^\top (\mathbf{\hat{B}} \mathbf{\hat{B}}^\top -  \mathbf{I}_d) \mathbf{X}_i^\top \mathbf{X}_i  \mathbf{B}\|_2=  \left\| \mathbf{U}^\top \mathbf{V}  \right\|_2 &\leq 2 \max_{\mathbf{z},\mathbf{y}  \in \mathcal{N}_k}  \mathbf{z}^\top \left( \mathbf{U}^\top \mathbf{V} \right)\mathbf{y} \nonumber \\
    &= 2 \max_{\mathbf{z}, \mathbf{y} \in \mathcal{N}_k} \mathbf{z}^\top\left(\sum_{j=1}^m  \mathbf{u}_j  \mathbf{v}_j^\top\right)\mathbf{y} \nonumber \\
    &= 2 \max_{\mathbf{z}, \mathbf{y} \in \mathcal{N}_k} \sum_{j=1}^m \langle \mathbf{z},  \mathbf{u}_j\rangle \langle \mathbf{v}_j,\mathbf{y} \rangle \nonumber 
\end{align}
By definition of sub-gaussianity, $\langle \mathbf{z},  \mathbf{u}_j\rangle$ and $\langle \mathbf{v}_j,\mathbf{y} \rangle$ are sub-gaussian with norms $\frac{1}{\sqrt{m}} \|\mathbf{\hat{B}}^{\ast^\top}(\mathbf{\hat{B}}\mathbf{\hat{B}}^\top - \mathbf{I}_d)\|_2=\frac{1}{\sqrt{m}}\text{dist}(\mathbf{\hat{B}}, \mathbf{\hat{B}}^\ast)$ and $\frac{1}{\sqrt{m}}\|\mathbf{\hat{B}}\|_2 = \frac{1}{\sqrt{m}}$, respectively. Thus for all $j\in [m]$, $\langle \mathbf{z},  \mathbf{u}_j\rangle \langle \mathbf{v}_j,\mathbf{z} \rangle $ is sub-exponential with norm $\frac{c}{m} \text{dist}(\mathbf{\hat{B}}, \mathbf{\hat{B}}^\ast)$ for some absolute constant $c$.  
Note that for any $j \in [m]$ and any $\mathbf{z}$, $\mathbb{E}[\langle \mathbf{z},  \mathbf{u}_j\rangle \langle \mathbf{v}_j,\mathbf{y} \rangle] = \mathbf{z}^\top ((\mathbf{\hat{B}}^\ast)^\top (\mathbf{\hat{B}} \mathbf{\hat{B}}^\top -  \mathbf{I}_d) \mathbf{B}) \mathbf{y} = 0$. Thus we have a sum of $m$ mean-zero, independent sub-exponential random variables. We can now use Bernstein's inequality to obtain, for any fixed $\mathbf{z}, \mathbf{y} \in \mathcal{N}_k$,
\begin{align}
    \mathbb{P}\left(\sum_{j=1}^m \langle \mathbf{z},  \mathbf{u}_j\rangle \langle \mathbf{v}_j,\mathbf{y} \rangle \geq s  \right) \leq \exp \left(-c' m \min\left(\frac{s^2}{\text{dist}^2(\mathbf{\hat{B}}, \mathbf{\hat{B}}^\ast) }, \frac{s}{\text{dist}(\mathbf{\hat{B}}, \mathbf{\hat{B}}^\ast)} \right)\right)
\end{align}
Now union bound over all $\mathbf{z}, \mathbf{y} \in \mathcal{N}_k$ to obtain
\begin{align}
\mathbb{P}\left( \frac{1}{m}\|(\mathbf{\hat{B}}^\ast)^\top (\mathbf{\hat{B}} \mathbf{\hat{B}}^\top -  \mathbf{I}_d) \mathbf{X}_i^\top \mathbf{X}_i  \mathbf{\hat{B}}\|_2 \geq 2s \right) &\leq 9^{2k} \exp \left(-c' m \min(s^2/\text{dist}^2(\mathbf{\hat{B}}, \mathbf{\hat{B}}^\ast) , s/\text{dist}(\mathbf{\hat{B}}, \mathbf{\hat{B}}^\ast) )\right)
\end{align}
Let $\frac{s}{\text{dist}(\mathbf{\hat{B}}, \mathbf{\hat{B}}^\ast)} = \max(\varepsilon, \varepsilon^2)$ for some $\epsilon > 0$, then it follows that $\min(s^2/\text{dist}^2(\mathbf{\hat{B}},\mathbf{\hat{B}}^\ast), s/\text{dist}(\mathbf{\hat{B}},\mathbf{\hat{B}}^\ast)) = \varepsilon^2$. So we have
\begin{align}
    \mathbb{P}\left(\frac{1}{m} \|(\mathbf{\hat{B}}^\ast)^\top (\mathbf{\hat{B}} \mathbf{\hat{B}}^\top -  \mathbf{I}_d) \mathbf{X}_i^\top \mathbf{X}_i  \mathbf{\hat{B}}\|_2 \geq 2\text{dist}(\mathbf{\hat{B}}, \mathbf{\hat{B}}^\ast)\max(\varepsilon, \varepsilon^2) \right) &\leq 9^{2k} e^{-c' m \varepsilon^2 }
\end{align}
Moreover, letting $\varepsilon^2 = \frac{ck^2\log(rn)}{4m}$ for some constant $c$, and $m \geq ck^2\log(rn)$, we have
\begin{align}
    \mathbb{P}\left( \frac{1}{m}\|(\mathbf{\hat{B}}^\ast)^\top (\mathbf{\hat{B}} \mathbf{\hat{B}}^\top -  \mathbf{I}_d) \mathbf{X}_i^\top \mathbf{X}_i  \mathbf{\hat{B}}\|_2 \geq \text{dist}(\mathbf{\hat{B}}, \mathbf{\hat{B}}^\ast) \sqrt{\frac{ck^2\log(rn)}{m}} \right)&\leq 9^{2k} e^{-c_1 k^2 
    \log(rn)}\nonumber \\
    &\leq e^{ - 111 k^2 
    \log(rn)}
\end{align}
for large enough constant $c_1$. Thus, noting that $\|\mathbf{H}^i \|_2^2 =\|\frac{1}{m}(\mathbf{\hat{B}}^\ast)^\top (\mathbf{\hat{B}} \mathbf{\hat{B}}^\top -  \mathbf{I}_d) \mathbf{X}_i^\top \mathbf{X}_i  \mathbf{\hat{B}}\|_2^2$, we obtain
\begin{align} \label{hiii}
    \mathbb{P}\left( \|\mathbf{H}^i\|_2^2 \geq c\dist^2(\mathbf{\hat{B}}, \mathbf{\hat{B}}^\ast) {\frac{k^2\log(rn)}{m}} \right)
    &\leq e^{ - 111 k^2 \log(rn)}
\end{align}
Thus, using \eqref{rw1}, we have
\begin{align}
   & \mathbb{P}\left( \|(\mathbf{GD} - \mathbf{C}) \mathbf{w}_\ast\|_2^2 \geq c\|\mathbf{W}^\ast\|_2^2 \; \dist^2(\mathbf{\hat{B}}, \mathbf{\hat{B}}^\ast) {\frac{k^3\log(rn)}{m}} \right)\nonumber\\
    &\leq \mathbb{P}\left( \frac{k}{rn} \|\mathbf{W}^\ast\|_2^2 \sum_{i=1}^{rn} \|\mathbf{H}^i\|_2^2 \geq c\|\mathbf{W}^\ast\|_2^2\; \text{dist}^2(\mathbf{\hat{B}}, \mathbf{\hat{B}}^\ast) {\frac{k^3\log(rn)}{m}} \right)  \nonumber \\
    &= \mathbb{P}\left( \frac{1}{rn}  \sum_{i=1}^{rn} \|\mathbf{H}^i\|_2^2 \geq c \dist^2(\mathbf{\hat{B}}, \mathbf{\hat{B}}^\ast) {\frac{k^2\log(rn)}{m}} \right)  \nonumber \\
    &\leq rn \mathbb{P}\left(  \|\mathbf{H}^1\|_2^2 \geq c \dist^2(\mathbf{\hat{B}}, \mathbf{\hat{B}}^\ast) {\frac{k^2\log(rn)}{m}} \right)\nonumber \\
    &\leq e^{ - 110 k^2 \log(rn)} \nonumber
\end{align}
completing the proof.
\end{proof}



\begin{lemma} \label{lem:bound_f}
Let $\delta_k = \frac{c k^{3/2} \sqrt{\log(rn)}}{\sqrt{m}} $, then
\begin{align}
    \| \mathbf{F} \|_F &\leq  \frac{\delta_k }{1-\delta_k} \|\mathbf{W}^\ast\|_2
    \;
    \dist(\mathbf{\hat{B}}_t, \mathbf{\hat{B}}_\ast)
\end{align}
 with probability at least $1 - e^{-110k^2 \log(n)}$.
\end{lemma}
\begin{proof}
By the definition of $\mathbf{F}$ and the Cauchy-Schwarz inequality, we have $\|\mathbf{F}\|_F = \|\mathbf{G}^{-1}(\mathbf{GD} - \mathbf{C})\widetilde{\mathbf{w}}^\ast\|_2 \leq\|\mathbf{G}^{-1}\|_2 \|(\mathbf{GD} - \mathbf{C})\widetilde{\mathbf{w}}^\ast\|_2$. Combining the bound on $\|\mathbf{G}^{-1}\|_2$ from Lemma \ref{lem:ginv} and the bound on $\|(\mathbf{GD} - \mathbf{C}) \widetilde{\mathbf{w}}^\ast\|_2$ from Lemma \ref{lem:gdif} via a union bound yields the result.

\end{proof}



We next focus on showing concentration of the operator $\frac{1}{m} \mathcal{A}^\dagger \mathcal{A}$ to the identity operator. 

\begin{lemma}\label{lemma:sth}
{Let $\delta_{k}'  = c k  \frac{\sqrt{d}}{\sqrt{rnm}}$ for some absolute constant $c$.} Then for any $t$, if $\delta_{k}' \leq k$,
\begin{align}
    \frac{1}{rn} \left\| \left(\frac{1}{m}\mathcal{A}^\ast \mathcal{A}(\mathbf{Q}^{t}) - \mathbf{Q}^{t}\right)^\top \mathbf{W}^{t+1}  \right\|_2 \leq {\delta_{k}'} \; \dist(\mathbf{\hat{B}}^t, \mathbf{\hat{B}}^\ast) 
\end{align}
with probability at least $1-e^{-110 d} - e^{-110k^2 \log(rn)}$.
\end{lemma}
\begin{proof}

We drop superscripts $t$ for simplicity. We first bound the norms of the rows of $\mathbf{Q}$ and $\mathbf{W}$. Let $\mathbf{q}_i \in \mathbb{R}^d$ be the $i$-th row of $\mathbf{Q}$ and let $\mathbf{w}_i \in \mathbb{R}^k$ be the $i$-th row of $\mathbf{W}$. 
Recall the computation of $\mathbf{W}$ from Lemma \ref{lem:defu}:
\begin{align}
    \mathbf{W} = {\mathbf{W}}_\ast  \mathbf{\hat{B}}_\ast^\top \mathbf{\hat{B}} - \mathbf{F}
    \implies \mathbf{w}_i^\top = (\hat{\mathbf{w}}_{i}^\ast)^\top \mathbf{\hat{B}}_\ast^\top \mathbf{\hat{B}}- \mathbf{f}_i^\top \nonumber
\end{align}
Thus
\begin{align}
    \|\mathbf{q}_i\|_2^2 
    &= \| \mathbf{\hat{B}} \mathbf{\hat{B}}^\top \mathbf{\hat{B}}^\ast \hat{\mathbf{w}}_{i}^\ast  - \mathbf{\hat{B}}\mathbf{f}_i - \mathbf{\hat{B}}^\ast \hat{\mathbf{w}}_{i}^\ast\|_2^2 \nonumber \\
    &= \| (\mathbf{\hat{B}} \mathbf{\hat{B}}^\top - \mathbf{I}_d)\mathbf{\hat{B}}^\ast \hat{\mathbf{w}}_{i}^\ast - \mathbf{\hat{B}}\mathbf{f}_i\|_2^2 \nonumber \\
    &\leq 2\| (\mathbf{\hat{B}} \mathbf{\hat{B}}^\top - \mathbf{I}_d)\mathbf{\hat{B}}^\ast \hat{\mathbf{w}}_{i}^\ast\|_2^2  + 2\| \mathbf{\hat{B}} \mathbf{f}_i\|_2^2 \nonumber \\
    &\leq 2\| (\mathbf{\hat{B}} \mathbf{\hat{B}}^\top - \mathbf{I}_d)\mathbf{\hat{B}}^\ast\|_2^2 \| \mathbf{\hat{w}}_{i}^\ast\|_2^2  + 2\| \mathbf{f}_i\|_2^2 \nonumber \\
    &=2k \text{dist}^2(\mathbf{\hat{B}}, \mathbf{\hat{B}}^\ast) + 2\| \mathbf{f}_i\|_2^2
\end{align}
Also recall that $\text{vec}(\mathbf{F}) =  \mathbf{G}^{-1}(\mathbf{GD} - \mathbf{C})  \mathbf{\hat{w}}_\ast$ from Lemma \ref{lem:defu}. From equation \eqref{fi}, the $i$-th row of $\mathbf{F}$ is given by:
\begin{align}
    \mathbf{f}_i = (\mathbf{G}^i)^{-1} (\mathbf{G}^i \mathbf{D}^i - \mathbf{C}^i){\mathbf{w}}^{\ast}_i \nonumber
\end{align}
Thus, using the Cauchy-Schwarz inequality and our previous bounds, 
\begin{align}
    \|\mathbf{f}_i\|_2^2 &\leq \| (\mathbf{G}^i)^{-1}\|_2^2\; \|\mathbf{G}^i \mathbf{D}^i - \mathbf{C}^i\|_2^2\;  \;\|{\mathbf{w}}^{\ast}_i\|_2^2 \nonumber \\
    &\leq \| (\mathbf{G}^i)^{-1}\|_2^2 \; \|\mathbf{G}^i \mathbf{D}^i - \mathbf{C}^i\|_2^2 \; k  
    \label{inco10}  
\end{align}
where \eqref{inco10} follows by Assumption \ref{assump:norm}, i.e. the row-wise incoherence of $\mathbf{W}^{\ast}$.
From \eqref{hiii}, we have that
\begin{align}
         \mathbb{P}\left( \|\mathbf{G}^i\mathbf{D}^i - \mathbf{C}^i\|_2^2 \geq \delta_k^2 \; \text{dist}^2(\mathbf{\hat{B}},\mathbf{\hat{B}}^\ast)  \right)
    &\leq e^{ - 112 k^2 \log(rn)} \nonumber 
\end{align}
where $\delta_k$ is defined in Lemma \ref{lem:ginv}.
Similarly, from equations \eqref{omd} and \eqref{omd_prob}, we have that 
\begin{align}
    \mathbb{P}\left(\|(\mathbf{G}^i)^{-1}\|_2^2 \geq \frac{1}{(1-\delta_k)^2}\right) \leq e^{-121 k^3 \log(rn)} \label{g52}
\end{align}

Now plugging this back into \eqref{inco10} and assuming $\delta_k \leq \frac{1}{2}$, we obtain
\begin{align}
 \|\mathbf{q}_i\|_2^2
 &\leq 2k\; \text{dist}^2(\mathbf{\hat{B}}, \mathbf{\hat{B}}^\ast)\;\left(1+\frac{\delta_k^2}{(1-\delta_k)^2}\right) \leq 4k\; \text{dist}^2(\mathbf{\hat{B}}, \mathbf{\hat{B}}^\ast)\label{recal}
\end{align}
with probability at least $1-e^{-111k^2\log(rn)}$.
Likewise, to upper bound $\|\mathbf{w}_i\|_2$ we have
\begin{align}
    \|\mathbf{w}_i\|_2^2 &\leq 2\| \mathbf{\hat{B}}^\top\mathbf{\hat{B}}^{\ast}\mathbf{w}_{i}^\ast \|_2^2 + 2 \|\mathbf{f}_i\|_2^2 \nonumber \\
    &\leq 2 \| \mathbf{\hat{B}}^\top\mathbf{\hat{B}}^{\ast}\|_2^2 \| \mathbf{w}_{i}^\ast \|_2^2  + 2 \|\mathbf{f}_i\|_2^2 \nonumber \\
    &\leq 2 {k} + 2 \frac{\delta_k^2}{(1-\delta_k)^2}\text{dist}^2(\mathbf{\hat{B}}, \mathbf{\hat{B}}^{\ast}){k} \label{toppp} \\
    &\leq 4k \label{lasat}
\end{align}
where \eqref{toppp} holds with probability at least $1-e^{-111k^2\log(rn)}$ conditioning on the same event as in \eqref{recal}, and \eqref{lasat} holds almost surely as long as $\delta_{k} \leq 1/2$. For the rest of the proof we condition on the event $\mathcal{E} \coloneqq \cap_{i=1}^{rn}\left\{ \|\mathbf{q}_i\|_2^2 \leq 4k \dist^2(\mathbf{\hat{B}}, \mathbf{\hat{B}}^\ast)\cap \|\mathbf{w}_i\|_2^2 \leq 4k\right\}$, which holds with probability at least $1-e^{-110k^2 \log(rn)}$ by a union bound over $i\in[rn]$.
Observe that the matrix $\frac{1}{m}\mathcal{A}^\ast \mathcal{A}(\mathbf{Q}) - \mathbf{Q}$ can be re-written as
\begin{align}
\frac{1}{m}\mathcal{A}^\ast \mathcal{A}(\mathbf{Q}) - \mathbf{Q}&= \frac{1}{m}\sum_{i =1}^{rn} \sum_{j=1}^m \left(\langle \mathbf{e}_i (\mathbf{x}_{i}^j)^\top,  \mathbf{Q}\rangle \; \mathbf{e}_i (\mathbf{x}_{i}^j)^\top - \mathbf{Q}\right) \nonumber \\
&= \frac{1}{m} \sum_{i =1}^{rn}  \sum_{j=1}^m \langle \mathbf{x}_{i}^j,  \mathbf{q}_i\rangle \; \mathbf{e}_i (\mathbf{x}_{i}^j)^\top - \mathbf{Q}
\end{align}
Multiplying the transpose by  $\frac{1}{rn}\mathbf{W}$ yields
\begin{align}
\frac{1}{rn}\left(\frac{1}{m}\mathcal{A}^\ast \mathcal{A}(\mathbf{Q}) - \mathbf{Q}\right)^\top \mathbf{W} &= \frac{1}{rnm}\sum_{i =1}^n \sum_{j=1}^m \left( \langle  \mathbf{x}_{i}^j,\, \mathbf{q}_i\rangle \;  \mathbf{x}_{i}^j ( \mathbf{w}_i)^\top -  \mathbf{q}_i(\mathbf{w}_i)^\top\right)  \label{heree}
\end{align}
where we have used the fact that $(\mathbf{Q})^\top \mathbf{W} = \sum_{i=1}^n  \mathbf{q}_i(\mathbf{w}_i)^\top$. We will argue similarly as in Proposition 4.4.5 in \cite{vershynin2018high} to bound the spectral norm of the $d$-by-$k$ matrix in the RHS of \eqref{heree}. 

First, let $\mathcal{S}^{d-1}$ and $\mathcal{S}^{k-1}$ denote the unit spheres in $d$ and $k$ dimensions, respectively. 
Construct $\frac{1}{4}$-nets $\mathcal{N}_d$ and $\mathcal{N}_k$ over $\mathcal{S}^{d-1}$ and $\mathcal{S}^{k-1}$, respectively, such that $|\mathcal{N}_d| \leq 9^d$ and $|\mathcal{N}_k| \leq 9^k$ (which is possible by Corollary 4.2.13 in \cite{vershynin2018high}). Then, using equation 4.13 in \cite{vershynin2018high}, we have
\begin{align}
&\left\|  \frac{1}{rnm}  \sum_{i =1}^{rn} \sum_{j=1}^m \left( \langle \mathbf{x}_{i}^j,  \mathbf{q}_i \rangle  \; \mathbf{x}_{i}^j (\mathbf{w}_i)^\top -  \mathbf{q}_i(\mathbf{w}_i)^\top\right)   \right\|_2^2 \nonumber \\
&\leq 2 \max_{\mathbf{u}\in \mathcal{N}_{d}, \mathbf{v}\in \mathcal{N}_{k} } \mathbf{u}^\top \left(  \sum_{i =1}^{rn} \sum_{j=1}^m \left( \frac{1}{rnm}\langle \mathbf{x}_{i}^j,  \mathbf{q}_i\rangle \; \mathbf{x}_{i}^j (\mathbf{w}_i)^\top - \frac{1}{rnm} \mathbf{q}_i (\mathbf{w}_i)^\top\right)   \right) \mathbf{v} \nonumber \\
&= 2 \max_{\mathbf{u}\in \mathcal{N}_{d}, \mathbf{v}\in \mathcal{N}_{k} }  \sum_{i =1}^{rn} \sum_{j=1}^m \left( \frac{1}{rnm}\langle \mathbf{x}_{i}^j,  \mathbf{q}_i \rangle \langle \mathbf{u} , \mathbf{x}_{i}^j \rangle \langle \mathbf{w}_i, \mathbf{v} \rangle - \frac{1}{rnm} \langle \mathbf{u}, \mathbf{q}_i\rangle \langle \mathbf{w}_i, \mathbf{v}\rangle  \right) 
\end{align}
By the $\mathbf{I}_d$-sub-gaussianity of $\mathbf{x}_{i}^j$,  the inner product $\langle \mathbf{u}, \mathbf{x}_{i}^j \rangle$ is sub-gaussian with norm at most $ c\| \mathbf{u} \|_2 = c$ for some absolute constant $c$ for any fixed $\mathbf{u} \in \mathcal{N}_d$. Similarly, $\langle \mathbf{x}_{i}^j, \mathbf{q}_i \rangle$ is sub-gaussian with norm at most $\|\mathbf{q}_i\|_2 \leq 2c \sqrt{k}\; \text{dist}(\mathbf{\hat{B}}, \mathbf{\hat{B}}^\ast)$ 
using \eqref{recal}. Further, since the sub-exponential norm of the product of two sub-gaussian random variables is at most  the product of the sub-gaussian norms of the two random variables (Lemma 2.7.7 in \cite{vershynin2018high}), we have that
$
 \langle \mathbf{x}_{i}^j,  \mathbf{q}_i\rangle \langle  \mathbf{u} , \mathbf{x}_{i}^j \rangle
$
is sub-exponential with norm at most $ 2c^2 \sqrt{k}\; \text{dist}(\mathbf{\hat{B}}, \mathbf{\hat{B}}^\ast)$. 
Further, $
 \frac{1}{rnm}\langle \mathbf{x}_{i}^j,  \mathbf{q}_i\rangle \langle  \mathbf{u} , \mathbf{x}_{i}^j \rangle \langle \mathbf{w}_i, \mathbf{v} \rangle
$ is sub-exponential with norm at most $$
 \frac{2c^2\sqrt{k}}{rnm} \; \text{dist}(\mathbf{\hat{B}}, \mathbf{\hat{B}}^\ast) \langle \mathbf{w}_i, \mathbf{v} \rangle \leq  \frac{2c^2\sqrt{k}}{rnm} \; \text{dist}(\mathbf{\hat{B}}, \mathbf{\hat{B}}^\ast) \|\mathbf{w}_i\|_2 \leq \frac{c_1{k}}{rnm} \; \text{dist}(\mathbf{\hat{B}}, \mathbf{\hat{B}}^\ast).
$$
Finally, note that $\mathbb{E}[\frac{1}{rnm}\langle \mathbf{x}_{i}^j,  \mathbf{q}_i \rangle \langle \mathbf{u} , \mathbf{x}_{i}^j \rangle \langle \mathbf{w}_i, \mathbf{v} \rangle - \frac{1}{rnm} \langle \mathbf{u}, \mathbf{q}_i\rangle \langle \mathbf{w}_i, \mathbf{v}\rangle ] = 0$. Thus, we have a sum of $rnm$ independent, mean zero sub-exponential random variables, so we apply Bernstein's inequality.
\begin{align}
&    \mathbb{P} \left(\sum_{i =1}^{rn} \sum_{j=1}^m \left( \frac{1}{rnm}\langle \mathbf{x}_{i}^j,  \mathbf{q}_i \rangle \langle \mathbf{u} , \mathbf{x}_{i}^j \rangle \langle \mathbf{w}_i, \mathbf{v} \rangle - \frac{1}{rnm} \langle \mathbf{u}, \mathbf{q}_i\rangle \langle \mathbf{w}_i, \mathbf{v}\rangle  \right) \geq s \right) \nonumber\\
    &\leq \exp\left(-c_1 rnm \min\left( \frac{s^2}{k^2 \; \text{dist}^2(\mathbf{\hat{B}}, \mathbf{\hat{B}}^\ast)}, \frac{s}{k\text{dist}(\mathbf{\hat{B}}, \mathbf{\hat{B}}^\ast)} \right) \right) \nonumber 
\end{align}
Union bounding over all $\mathbf{u} \in \mathcal{N}_d$ and $\mathbf{v} \in \mathcal{N}_k$, we obtain
\begin{align}
    \mathbb{P} \left( \left\| \frac{1}{rn}\left(\frac{1}{m}\mathcal{A}^\ast \mathcal{A}(\mathbf{Q}) - \mathbf{Q}\right)^\top \mathbf{W} \right\|_2 \geq 2s \; \Big|\;  \mathcal{E} \right) 
    &\leq 9^{d+k}\exp\left(-c_1 rnm \min\left( \frac{s^2}{k^2 \; \text{dist}^2(\mathbf{\hat{B}}, \mathbf{\hat{B}}^\ast)}, \frac{s}{k\text{dist}(\mathbf{\hat{B}}, \mathbf{\hat{B}}^\ast)} \right) \right) \nonumber 
\end{align}
Let $\frac{s}{k\; \text{dist}(\mathbf{\hat{B}}, \mathbf{\hat{B}}^\ast)} = \max{(\epsilon, \epsilon^2)}$ for some $\epsilon >0$, then $\epsilon^2 = \min\left( \frac{s^2}{k^2\;\text{dist}^2(\mathbf{\hat{B}}, \mathbf{\hat{B}}^\ast)}, \frac{s}{k\; \text{dist}(\mathbf{\hat{B}}, \mathbf{\hat{B}}^\ast)}\right)$. Further, let $\epsilon^2 = \frac{112(d+k)}{c_1 rnm} $, then as long as $\epsilon^2 \leq 1$, we have
\begin{align}
    \mathbb{P} \left( \left\| \frac{1}{rn}\left(\frac{1}{m}\mathcal{A}^\ast \mathcal{A}(\mathbf{Q}) - \mathbf{Q}\right)^\top \mathbf{W} \right\|_2 \geq c_2 k\; \text{dist}(\mathbf{\hat{B}}, \mathbf{\hat{B}}^\ast) \sqrt{d/(rnm)} \; \Big| \; \mathcal{E}^c \right) 
    &\leq e^{-110(d+k)}\leq e^{-110d}.  \nonumber 
\end{align}
Finally, we use $\mathbb{P} \left( A \; | \; \mathcal{E}^c \right) \leq \mathbb{P} \left( A \; | \; \mathcal{E}^c \right) + \mathbb{P}(\mathcal{E}^c)$, where \\
$A\coloneqq \left\{\left\| \frac{1}{rn}\left(\frac{1}{m}\mathcal{A}^\ast \mathcal{A}(\mathbf{Q}) - \mathbf{Q}\right)^\top \mathbf{W} \right\|_2 \geq c_2 k\; \text{dist}(\mathbf{\hat{B}}, \mathbf{\hat{B}}^\ast) \sqrt{d/(rnm)}\right\}$, to complete the proof.

\end{proof}





\subsection{Main Result}

Now we are ready to show Theorem \ref{thrm:linear}, which follows immediately from the following descent lemma.

\begin{lemma} \label{lem:main}
Define $E_0 \coloneqq 1- \dist^2(\mathbf{\hat{B}}^0,\mathbf{\hat{B}}^\ast)$ and $\bar{\sigma}_{\max, \ast} \coloneqq \max_{\mathcal{I} \in [n], |\mathcal{I}| = rn} \sigma_{\max} (\frac{1}{\sqrt{rn}}\mathbf{W}^\ast_{\mathcal{I}})$ and $\bar{\sigma}_{\min, \ast} \coloneqq \min_{\mathcal{I} \in [n], |\mathcal{I}| = rn} \sigma_{\min}(\frac{1}{\sqrt{rn}} \mathbf{W}^\ast_{\mathcal{I}})$, i.e. the maximum and minimum singular values of any matrix that can be obtained by taking $rn$ rows of $\frac{1}{\sqrt{rn}}\mathbf{W}^\ast$.

Suppose that $m \geq c(\kappa^4 k^3\log(rn)/E_0^2 + \kappa^4 k^2 d/(E_0^2 rn))$ for some absolute constant $c$. 
Then for any $t$ and any $\eta \leq 1/(4 \bar{\sigma}_{\max,\ast}^2)$, we have 
\begin{align}
    \dist(\mathbf{\hat{B}}^{t+1}, \mathbf{\hat{B}}^\ast)
    &\leq \left(1 -  \eta E_0\bar{\sigma}_{\min,\ast}^2 /2\right)^{1/2}\; \dist(\mathbf{\hat{B}}^t,\mathbf{\hat{B}}^\ast), \nonumber
\end{align}
with probability at least $1 - e^{-100\min(k^2\log(rn),d)}$.
\end{lemma}

\begin{proof}


Recall that $\mathbf{W}^{t+1} \in\mathbb{R}^{rn \times k}$ and $\mathbf{\bar{B}}^{t+1} \in\mathbb{R}^{d \times k}$ are computed as follows:
\begin{align}
    {\mathbf{{W}}}^{t+1} &= 
    \argmin_{{\mathbf{{W}}} \in \mathbb{R}^{rn\times k}} \frac{1}{2rnm}\| \mathcal{A}({\mathbf{{W}}}^\ast\mathbf{\hat{B}}^{\ast^\top} -  {\mathbf{{W}}}\mathbf{\hat{B}}^{t^\top})\|_2^2 \label{up_w} \\
  \mathbf{\bar{B}}^{t+1} &= \mathbf{\hat{B}}^t - \!\frac{\eta}{rnm} 
  \left(\mathcal{A}^{\dagger} \mathcal{A}({\mathbf{W}}^{t+1}\mathbf{\hat{B}}^{t^\top}- {\mathbf{{W}}}^\ast\mathbf{\hat{B}}^{\ast^\top})\right)^\top {\mathbf{{W}}}^{t+1}
\end{align}

Let 
$\mathbf{Q}^t = \mathbf{W}^{t+1} \mathbf{\hat{B}}^{t^\top} - \mathbf{{W}}^\ast \mathbf{\hat{B}}^{\ast^\top}$. We have
\begin{align}
    \mathbf{\bar{B}}^{t+1} &= \mathbf{\hat{B}}^t -  \frac{\eta}{rnm} \left(\mathcal{A}^\dagger \mathcal{A}(\mathbf{Q}^t ) \right)^\top \mathbf{W}^{t+1} \nonumber\\
    & =  \mathbf{\hat{B}}^t -  \frac{\eta}{rn} \; \mathbf{Q}^{t^\top}\mathbf{W}^{t+1} - \frac{\eta}{rn}\left(\frac{1}{m}\mathcal{A}^\dagger \mathcal{A}(\mathbf{Q}^t ) - \mathbf{Q}^t \right)^\top \mathbf{W}^{t+1} 
\end{align}
Now, multiply both sides by  $\mathbf{\hat{B}}_{\perp}^{\ast^\top}$ to obtain 
\begin{align}
     \mathbf{\hat{B}}_{\perp}^{\ast^\top} \mathbf{\bar{B}}^{t+1}
     &= \mathbf{\hat{B}}_{\perp}^{\ast^\top} \mathbf{\hat{B}}^t -  \frac{\eta}{rn}  \mathbf{\hat{B}}_{\perp}^{\ast^\top} \mathbf{Q}^{t^\top}\mathbf{W}^{t+1} - \frac{\eta}{rn}  \mathbf{\hat{B}}_{\perp}^{\ast^\top} \left(\frac{1}{m}\mathcal{A}^\dagger \mathcal{A}(\mathbf{Q}^t ) - \mathbf{Q}^t \right)^\top \mathbf{W}^{t+1} \nonumber\\
     &= \mathbf{\hat{B}}_{\perp}^{\ast^\top} \mathbf{\hat{B}}^t (\mathbf{I}_k - \frac{\eta}{rn} \mathbf{W}^{{t+1}^\top} \mathbf{W}^{t+1}) - \frac{\eta}{rn} \mathbf{\hat{B}}_{\perp}^{\ast^\top} \left(\frac{1}{m}\mathcal{A}^\dagger \mathcal{A}(\mathbf{Q}^t ) - \mathbf{Q}^t \right)^\top \mathbf{W}^{t+1} \label{beforeR}
     \end{align}
where the second equality follows because $\mathbf{\hat{B}}_{\perp}^{\ast^\top} \mathbf{Q}^{t^\top} = \mathbf{\hat{B}}_{\perp}^{\ast^\top} \mathbf{\hat{B}}^t \mathbf{W}^{{t+1}^\top} - \mathbf{\hat{B}}_{\perp}^{\ast^\top} \mathbf{\hat{B}}^\ast \mathbf{{W}}^{\ast^\top} = \mathbf{\hat{B}}_{\perp}^{\ast^\top} \mathbf{\hat{B}}^t \mathbf{W}^{{t+1}^\top}$.
Then, writing the QR decomposition of $\mathbf{\bar{B}}^{t+1}$ as $\mathbf{{B}}^{t+1} = \mathbf{\hat{B}}^{t+1} \mathbf{R}^{t+1}$ and multiplying both sides of \eqref{beforeR} from the right by $(\mathbf{R}^{t+1})^{-1}$ yields
     \begin{align}
    \mathbf{\hat{B}}_{\perp}^{\ast^\top} \mathbf{\hat{B}}^{t+1} =  \left(\mathbf{\hat{B}}_{\perp}^{\ast^\top} \mathbf{\hat{B}}^t (\mathbf{I}_k - \frac{\eta}{rn} (\mathbf{W}^{t+1})^\top \mathbf{W}^{t+1}) - \frac{\eta}{rn} \mathbf{\hat{B}}_{\perp}^{\ast^\top} \left(\frac{1}{m}\mathcal{A}^\dagger \mathcal{A}(\mathbf{Q}^t ) - \mathbf{Q}^t \right)^\top \mathbf{W}^{t+1}\right)(\mathbf{R}^{{t+1}})^{-1} \end{align}
    Hence,
      \begin{align}
 &\text{dist}(\mathbf{\hat{B}}^{t+1}, \mathbf{\hat{{B}}}^{\ast}) \nonumber\\
    &= \left\|\left(\mathbf{\hat{B}}_{\perp}^{\ast^\top} \mathbf{\hat{B}}^t (\mathbf{I}_k - \frac{\eta}{rn} (\mathbf{W}^{t+1})^\top \mathbf{W}^{t+1}) - \frac{\eta}{rn} \mathbf{\hat{B}}_{\perp}^{\ast^\top} \left(\frac{1}{m}\mathcal{A}^\dagger \mathcal{A}(\mathbf{Q}^t ) - \mathbf{Q}^t \right)^\top \mathbf{W}^{t+1}\right)(\mathbf{R}^{{t+1}})^{-1}\right\|_2 \nonumber \\
    &\leq \left\| \mathbf{\hat{B}}_{\perp}^{\ast^\top} \mathbf{\hat{B}}^t (\mathbf{I}_k - \frac{\eta}{rn} (\mathbf{W}^{t+1})^\top \mathbf{W}^{t+1})\right\|_2 \left\|(\mathbf{R}^{t+1})^{-1}\right\|_2 \nonumber\\
    &\qquad + \frac{\eta}{rn} \left\|\mathbf{\hat{B}}_{\perp}^{\ast^\top} \left(\frac{1}{m}\mathcal{A}^\dagger \mathcal{A}(\mathbf{Q}^t ) - \mathbf{Q}^t \right)^\top \mathbf{W}^{t+1} \right\|_2 \left\|(\mathbf{R}^{t+1})^{-1}\right\|_2 \label{ttwoterms} \\
    &=: A_1 + A_2  .
\end{align}
where \eqref{ttwoterms} follows by applying the triangle and Cauchy-Schwarz inequalities.
We have thus split the upper bound on $\text{dist}(\mathbf{B}^{t+1}, \mathbf{\hat{B}}^{\ast})$ into two terms, $A_1$ and $A_2$. The second term, $A_2$, is small due to the concentration of $\frac{1}{m}\mathcal{A}^\dagger \mathcal{A}$ to the identity operator, and the first term is strictly smaller than $\text{dist}(\hat{\mathbf{B}}^{t}, \hat{\mathbf{{B}}}^{\ast})$. We start by controlling $A_2$:
\begin{align}
  A_2 &= \frac{\eta}{rn} \left\|\mathbf{\hat{B}}_{\perp}^{\ast^\top} \left(\frac{1}{m}\mathcal{A}^\dagger \mathcal{A}(\mathbf{Q}^t ) - \mathbf{Q}^t \right)^\top \mathbf{W}^{t+1} \right\|_2 \left\|(\mathbf{R}^{t+1})^{-1}\right\|_2 \nonumber \\
   &\leq \frac{\eta}{rn} \left\| \left(\frac{1}{m}\mathcal{A}^\dagger \mathcal{A}(\mathbf{Q}^t ) - \mathbf{Q}^t \right)^\top \mathbf{W}^{t+1} \right\|_2 \left\|(\mathbf{R}^{t+1})^{-1}\right\|_2 \label{csss} \\
   &\leq {\eta}{{\delta_{k}'}}\; \text{dist}(\mathbf{\hat{B}}^t, \mathbf{\hat{B}}^\ast) \; \|(\mathbf{R}^{t+1})^{-1}\|_2  \label{qq}
\end{align}
where \eqref{csss} follows almost surely by Cauchy-Schwarz and the fact that $\mathbf{\hat{B}}^\ast_\perp$ is normalized, and \eqref{qq} follows with probability at least $1 - e^{-110d}$ by Lemma \ref{lemma:sth}.
Next we control $A_1$:
\begin{align}
 A_1 &= \left\| \mathbf{\hat{B}}_{\perp}^{\ast^\top} \mathbf{\hat{B}}^t (\mathbf{I}_k - \frac{\eta}{rn} (\mathbf{W}^{t+1})^\top \mathbf{W}^{t+1})\right\|_2 \|(\mathbf{R}^{t+1})^{-1}\|_2  \nonumber \\
 &\leq \| \mathbf{\hat{B}}_{\perp}^{\ast^\top} \mathbf{\hat{B}}^t \|_2 \left\|\mathbf{I} - \frac{\eta}{rn} (\mathbf{W}^{t+1})^\top \mathbf{W}^{t+1}\right\|_2 \|(\mathbf{R}^{t+1})^{-1}\|_2   \nonumber \\
    &= \text{dist}(\mathbf{\hat{B}}^t, \mathbf{\hat{B}}^\ast)\;  \left\|\mathbf{I}_k - \frac{\eta}{rn} (\mathbf{W}^{t+1})^\top \mathbf{W}^{t+1}\right\|_2\; \|(\mathbf{R}^{t+1})^{-1}\|_2  \label{middle} 
    \end{align}
The middle factor gives us contraction. To see this, recall that $\mathbf{W}^{t+1}= \mathbf{W}^\ast \hat{\mathbf{B}}^{\ast^\top} \mathbf{\hat{B}}^t - \mathbf{F}$ where $\mathbf{F}$ is defined in Lemma \ref{lem:defu}.
By Lemma \ref{lem:bound_f}, we have that
\begin{align}
    \|\mathbf{F}\|_2 &\leq \frac{\delta_k}{1-\delta_k} \|\mathbf{W}^\ast\|_2 \; \text{dist}(\mathbf{\hat{B}}^t,\mathbf{\hat{B}}^\ast)
\end{align}
with probability at least $1- e^{-110k^2 \log(rn)}$, which we will use throughout the proof. Conditioning on this event, we have
\begin{align}
    \lambda_{\max}\left((\mathbf{W}^{t+1})^\top \mathbf{W}^{t+1}\right) &= \|\mathbf{W}^\ast \hat{\mathbf{B}}^{\ast^\top} \mathbf{\hat{B}}^t - \mathbf{F}\|_2^2 \nonumber \\
    &\leq 2\|\mathbf{W}^\ast \hat{\mathbf{B}}^{\ast^\top} \mathbf{\hat{B}}^t \|_2^2 +2\| \mathbf{F}\|_2^2 \nonumber \\
    &\leq 2\|\mathbf{W}^\ast \|_2^2 +2\frac{\delta_k^2}{(1-\delta_k)^2}\|\mathbf{W}^\ast\|_2^2 \; \text{dist}^2(\mathbf{\hat{B}}^t, \mathbf{\hat{B}}^\ast) \nonumber \\
    &\leq 4\|\mathbf{W}^\ast \|_2^2 \label{triv}
\end{align}
where \eqref{triv} follows under the assumption that $\delta_k \leq 1/2$. Thus, as long as $\eta \leq 1/(4 \bar{\sigma}_{\max,\ast}^2)$, we have by Weyl's Inequality:
\begin{align}
&\|\mathbf{I}_k - \frac{\eta}{rn} (\mathbf{W}^{t+1})^\top \mathbf{W}^{t+1}\|_2
\nonumber\\
&\leq  1 - \frac{\eta}{rn} \lambda_{\min} ((\mathbf{W}^{t+1})^\top \mathbf{W}^{t+1})  \\
    &= 1 - \frac{\eta}{rn} \lambda_{\min} (({\mathbf{W}}^\ast \mathbf{\hat{B}}^{\ast^\top} \mathbf{\hat{B}}^t - \mathbf{F})^\top ({\mathbf{W}}^\ast \mathbf{\hat{B}}^{\ast^\top} \mathbf{\hat{B}}^t - \mathbf{F})) \nonumber \\
    &\leq 1 - \frac{\eta}{rn} \sigma_{\min}^2({\mathbf{W}}^\ast (\mathbf{\hat{B}}^\ast)^\top \mathbf{\hat{B}}^t) +\frac{2\eta}{rn} \sigma_{\max}(\mathbf{F}^\top{\mathbf{W}}^\ast (\mathbf{\hat{B}}^\ast)^\top \mathbf{\hat{B}}^t) - \frac{\eta}{rn} \sigma^2_{\min}( \mathbf{F}) \label{wyl}\\
    &\leq 1 - \frac{\eta}{rn} \sigma_{\min}^2({\mathbf{W}}^\ast)\sigma^2_{\min}( (\mathbf{\hat{B}}^\ast)^\top \mathbf{\hat{B}}^t) +\frac{2\eta}{rn} \|\mathbf{F}\|_2\;\|{\mathbf{W}}^\ast (\mathbf{\hat{B}}^\ast)^\top \mathbf{\hat{B}}^t\|_2  \label{usef} \\
    &\leq 1 - \frac{\eta}{rn} \sigma_{\min}^2({\mathbf{W}}^\ast) \sigma^2_{\min}( (\mathbf{\hat{B}}^\ast)^\top\mathbf{ \hat{B}}^t)  + \frac{2\eta}{rn} \frac{\delta_{k}}{1-\delta_{k}} \|\mathbf{W}^{\ast}\|_2^2 \label{last} \\
    &= 1 - {\eta} \bar{\sigma}_{\min,\ast}^2 \sigma^2_{\min}( (\mathbf{\hat{B}}^\ast)^\top \mathbf{\hat{B}}^t)  + {2\eta} \frac{\delta_{k}}{1-\delta_{k}} \bar{\sigma}^2_{\max, \ast} \label{defsig}
\end{align}
where \eqref{wyl} follows by again applying Weyl's inequality, under the condition that \\
$2\sigma_{\max}(\mathbf{F}^\top{\mathbf{W}}^\ast (\mathbf{\hat{B}}^\ast)^\top \mathbf{\hat{B}}^t) \leq \sigma_{\min}^2({\mathbf{W}}^\ast)\sigma^2_{\min}( (\mathbf{\hat{B}}^\ast)^\top \mathbf{\hat{B}}^t) $, which we will enforce to be true (otherwise we would not have contraction). Also, \eqref{usef} follows by the Cauchy-Schwarz inequality, and we use Lemma \ref{lem:bound_f} to obtain \eqref{last}. Lastly, \eqref{defsig} follows by the definitions of $\bar{\sigma}_{\min,\ast}$ and $\bar{\sigma}_{\max,\ast}$. In order to lower bound $\sigma^2_{\min}( (\mathbf{\hat{B}}^\ast)^\top \mathbf{\hat{B}}^t)$, note that
\begin{align}
    \sigma^2_{\min}( (\mathbf{\hat{B}}^\ast)^\top \mathbf{\hat{B}}^t) &\geq 1 - \|(\mathbf{\hat{B}}_\perp^\ast)^\top \mathbf{\hat{B}}^t\|^2_2 = 1 - \text{dist}^2(\mathbf{\hat{B}}^t, \mathbf{\hat{B}}^\ast) \geq 1 - \text{dist}^2(\mathbf{\hat{B}}^0, \mathbf{\hat{B}}^\ast) =: E_0 \label{32}
\end{align}
As a result, defining $\bar{\delta}_k \coloneqq \delta_k + \delta'_k$ and combining \eqref{ttwoterms}, \eqref{qq}, \eqref{middle}, \eqref{defsig}, and \eqref{32} yields
\begin{align}
    \text{dist}(\mathbf{\hat{B}}^{t+1}, \mathbf{\hat{B}}^\ast)
    &\leq \|(\mathbf{R}^{t+1})^{-1}\|_2\; (1 - {\eta} \bar{\sigma}_{\min,\ast}^2 E_0 + {2\eta} \frac{\delta_{k}}{1-\delta_{k}}\bar{\sigma}^2_{\max,\ast} +{\eta{\delta_{k}'}} )\; \text{dist}(\mathbf{\hat{B}}^t,\mathbf{\hat{B}}^\ast) \nonumber \\
    &\leq \|(\mathbf{R}^{t+1})^{-1}\|_2\; (1 - {\eta}\bar{\sigma}_{\min,\ast}^2 E_0 + {2\eta} \frac{\bar{\delta}_k}{1-\bar{\delta}_{k}}\bar{\sigma}^2_{\max,\ast} )\; \text{dist}(\mathbf{\hat{B}}^t,\mathbf{\hat{B}}^\ast) \label{justR}
\end{align}
where \eqref{justR} follows from the fact that $krn = \|\mathbf{W}^\ast\|_F^2 \leq k \|\mathbf{W}^\ast\|_2^2 \implies 1 \leq \|\mathbf{W}^\ast\|_2^2/rn \leq \bar{\sigma}_{\max,\ast}^2$.
All that remains to bound is $\|(\mathbf{R}^{t+1})^{-1}\|_2$. Define $\mathbf{S}^t \coloneqq \frac{1}{m}\mathcal{A}^\dagger \mathcal{A}(\mathbf{Q}^t)$ and observe that
\begin{align}
   (\mathbf{R}^{t+1})^\top \mathbf{R}^{t+1}
   &= (\mathbf{\bar{B}}^{t+1})^\top \mathbf{\bar{B}}^{t+1} \nonumber\\
   &= \mathbf{\hat{B}}^{t^\top} \mathbf{\hat{B}}^t - \frac{\eta}{rn}(\mathbf{\hat{B}}^{t^\top} \mathbf{S}^{t^\top} \mathbf{W}^{t+1} + (\mathbf{W}^{t+1})^\top \mathbf{S}^t \mathbf{\hat{B}}^t) + \frac{\eta^2}{(rn)^2} (\mathbf{W}^{t+1})^\top \mathbf{S}^t \mathbf{S}^{t^\top} \mathbf{W}^{t+1} \nonumber \\
    &=  \mathbf{I}_k - \frac{\eta}{rn}(\mathbf{\hat{B}}^{t^\top} \mathbf{S}^{t^\top} \mathbf{W}^{t+1} + (\mathbf{W}^{t+1})^\top \mathbf{S}^t \mathbf{\hat{B}}^t) + \frac{\eta^2}{(rn)^2} (\mathbf{W}^{t+1})^\top \mathbf{S}^t \mathbf{S}^{t^\top} \mathbf{W}^{t+1}
\end{align}
thus, by Weyl's Inequality, we have
\begin{align}
    \sigma_{\min}^2( \mathbf{R}_{t+1}) &\geq 1 - \frac{\eta}{rn} \lambda_{\max}(\mathbf{\hat{B}}^{t^\top} \mathbf{S}^{t^\top} \mathbf{W}^{t+1} + (\mathbf{W}^{t+1})^\top \mathbf{S}^t \mathbf{\hat{B}}^t) + \frac{\eta^2}{(rn)^2} \lambda_{\min} ((\mathbf{W}^{t+1})^\top \mathbf{S}^t \mathbf{S}^{t^\top} \mathbf{W}^{t+1}) \nonumber \\ 
    &\geq  1 - \frac{\eta}{rn} \lambda_{\max}(\mathbf{\hat{B}}^{t^\top} \mathbf{S}^{t^\top} \mathbf{W}^{t+1} + (\mathbf{W}^{t+1})^\top \mathbf{S}^t \mathbf{\hat{B}}^t)  \label{psd0}
\end{align}
where \eqref{psd0} follows because $(\mathbf{W}^{t+1})^\top \mathbf{S}^t \mathbf{S}^{t^\top} \mathbf{W}^{t+1}$ is positive semi-definite. Next, note that
\begin{align}
    \frac{\eta}{rn}\lambda_{\max}&(\mathbf{\hat{B}}^{t^\top} \mathbf{S}^{t^\top} \mathbf{W}^{t+1} + (\mathbf{W}^{t+1})^\top \mathbf{S}^t \mathbf{\hat{B}}^t) \nonumber \\
    &= \max_{\mathbf{x}:\|\mathbf{x}\|_2=1} \frac{\eta}{rn}\mathbf{x}^\top \mathbf{\hat{B}}^{t^\top} (\mathbf{S}^t)^\top \mathbf{W}^{t+1}\mathbf{x}+ \mathbf{x}^\top (\mathbf{W}^{t+1})^\top \mathbf{S}^t \mathbf{\hat{B}}^t\mathbf{x}\nonumber \\
    &= \max_{\mathbf{x}:\|\mathbf{x}\|_2=1} \frac{2\eta}{rn} \mathbf{x}^\top (\mathbf{W}^{t+1})^\top \mathbf{S}^t \mathbf{\hat{B}}^t\mathbf{x}\nonumber \\
    &= \max_{\mathbf{x}:\|\mathbf{x}\|_2=1} \frac{2\eta}{rn} \mathbf{x}^\top (\mathbf{W}^{t+1})^\top \left( \frac{1}{m}\mathcal{A}^\dagger \mathcal{A}(\mathbf{Q}^t) - \mathbf{Q}^t\right) \mathbf{\hat{B}}^t\mathbf{x} +  \frac{2\eta}{rn} \mathbf{x}^\top (\mathbf{W}^{t+1})^\top \mathbf{Q}^t \mathbf{\hat{B}}^t\mathbf{x}\label{2terms}
\end{align}
We first consider the first term. We have
\begin{align}
   \max_{\mathbf{x}:\|\mathbf{x}\|_2=1} \frac{2\eta}{rn} \mathbf{x}^\top (\mathbf{W}^{t+1})^\top \left(\frac{1}{m}\mathcal{A}^\dagger \mathcal{A}(\mathbf{Q}^t) - \mathbf{Q}^t\right) \mathbf{\hat{B}}^t\mathbf{x}  
  \leq \frac{2\eta}{rn} \left\|
   (\mathbf{W}^{t+1})^\top\left( \frac{1}{m}\mathcal{A}^\dagger \mathcal{A}(\mathbf{Q}^t) - \mathbf{Q}^t\right)\right\|_2\;\left\| \mathbf{\hat{B}}^{t} \right\|_2
   \leq  {2\eta}{\delta'_{k}} \label{ddd}  
\end{align}
where the last inequality follows with probability at least $1-e^{-110d} - e^{-110k^2\log(rn)}$ from Lemma \ref{lemma:sth}.
Next we turn to the second term in \eqref{2terms}. 
We have
\begin{align}
    \max_{\mathbf{x}:\|\mathbf{x}\|_2=1} \frac{2\eta}{rn} \mathbf{x}^\top (\mathbf{W}^{t+1})^\top \mathbf{Q}^t \mathbf{\hat{B}}^t\mathbf{x} &=\max_{\mathbf{x}:\|\mathbf{x}\|_2=1} \frac{2\eta}{rn} \left\langle  \mathbf{Q}^t, \mathbf{W}^{t+1}\mathbf{xx}^\top \mathbf{\hat{B}}^{t^\top} \right\rangle \nonumber \\
    &= \max_{\mathbf{x}:\|\mathbf{x}\|_2=1} \frac{2\eta}{rn} \langle \mathbf{Q}^t, \mathbf{{W}}^\ast  \mathbf{\hat{B}}^{\ast^\top} \mathbf{\hat{B}}^t  \mathbf{xx}^\top \mathbf{\hat{B}}^{t^\top} \rangle - \frac{2\eta}{rn} \langle \mathbf{Q}^t,  \mathbf{F}\mathbf{xx}^\top \mathbf{\hat{B}}^{t^\top} \rangle 
\end{align}
For any $\mathbf{x} \in \mathbb{R}^k: \|\mathbf{x}\|_2=1$, we have
\begin{align}
    &\frac{2\eta}{rn} \langle \mathbf{Q}^t, \mathbf{{W}}^\ast (\mathbf{\hat{B}}^\ast)^\top \mathbf{\hat{B}}^t \mathbf{xx}^\top \mathbf{\hat{B}}^{t^\top} \rangle\nonumber\\
    &= \frac{2\eta}{rn}\text{tr}( (\mathbf{\hat{B}}^t (\mathbf{W}^{t+1})^\top - \mathbf{\hat{B}}^\ast (\mathbf{{W}}^\ast)^\top) \mathbf{{W}}^\ast \mathbf{\hat{B}}^{\ast^\top} \mathbf{\hat{B}}^t \mathbf{xx}^\top \mathbf{\hat{B}}^{t^\top} )\nonumber \\
    &=  \frac{2\eta}{rn}\text{tr}( (\mathbf{\hat{B}}^t \mathbf{\hat{B}}^{t^\top} \mathbf{\hat{B}}^\ast  \mathbf{{W}}^{\ast^\top} - \mathbf{\hat{B}}^t \mathbf{F}^\top - \mathbf{\hat{B}}^\ast  \mathbf{{W}}^{\ast^\top}) \mathbf{{W}}^\ast  \mathbf{\hat{B}}^{\ast^\top} \mathbf{\hat{B}}^t \mathbf{xx}^\top \mathbf{\hat{B}}^{t^\top} )\nonumber \\
    &= \frac{2\eta}{rn} \text{tr}( (\mathbf{\hat{B}}^t \mathbf{\hat{B}}^{t^\top} -\mathbf{I})  \mathbf{\hat{B}}^{\ast^\top} \mathbf{{W}}^{\ast^\top} \mathbf{{W}}^\ast \mathbf{\hat{B}}^{\ast^\top} \mathbf{\hat{B}}^t \mathbf{xx}^\top \mathbf{\hat{B}}^{t^\top} )
    - \frac{2\eta}{rn}\text{tr}( \mathbf{\hat{B}}^t \mathbf{F}^\top \mathbf{{W}}^\ast  \mathbf{\hat{B}}^{\ast^\top} \mathbf{\hat{B}}^t \mathbf{xx}^\top \mathbf{\hat{B}}^{t^\top} ) \nonumber \\
    &= \frac{2\eta}{rn}\text{tr}( \mathbf{\hat{B}}_\perp^t  \mathbf{\hat{B}}^{\ast^\top} \mathbf{{W}}^{\ast^\top} \mathbf{{W}}^\ast \mathbf{\hat{B}}^{\ast^\top} \mathbf{\hat{B}}^t \mathbf{xx}^\top \mathbf{\hat{B}}^{t^\top} ) - \frac{2\eta}{rn}\text{tr}( \mathbf{\hat{B}}^t \mathbf{F}^\top \mathbf{{W}}^\ast  \mathbf{\hat{B}}^{\ast^\top} \mathbf{\hat{B}}^t \mathbf{xx}^\top \mathbf{\hat{B}}^{t^\top} ) \nonumber \\
    &= \frac{2\eta}{rn}\text{tr}(  \mathbf{\hat{B}}^{\ast^\top} \mathbf{{W}}^{\ast^\top} \mathbf{{W}}^\ast \mathbf{\hat{B}}^{\ast^\top} \mathbf{\hat{B}}^t \mathbf{xx}^\top \mathbf{\hat{B}}^{t^\top} \mathbf{\hat{B}}_\perp^t ) - \frac{2\eta}{rn}\text{tr}( \mathbf{\hat{B}}^t \mathbf{F}^\top \mathbf{{W}}^\ast  \mathbf{\hat{B}}^{\ast^\top} \mathbf{\hat{B}}^t \mathbf{xx}^\top \mathbf{\hat{B}}^{t^\top} ) \nonumber \\
    &= - \frac{2\eta}{rn}\text{tr}( \mathbf{F}^\top \mathbf{{W}}^\ast  \mathbf{\hat{B}}^{\ast^\top} \mathbf{\hat{B}}^t \mathbf{xx}^\top \mathbf{\hat{B}}^{t^\top} \mathbf{\hat{B}}^t) \label{perp0} \\
    &=  - \frac{2\eta}{rn}\text{tr}( \mathbf{F}^\top \mathbf{{W}}^\ast  \mathbf{\hat{B}}^{\ast^\top} \mathbf{\hat{B}}^t \mathbf{xx}^\top) \label{i0} \\
    &\leq \frac{2\eta}{rn}\|\mathbf{F}\|_F\;\| \mathbf{{W}}^\ast  \mathbf{\hat{B}}^{\ast^\top} \mathbf{\hat{B}}^t \mathbf{xx}^\top \|_F \label{cscs} \\
    &\leq \frac{2\eta}{rn}\|\mathbf{F}\|_F \| \mathbf{{W}}^\ast \|_2 \| \mathbf{\hat{B}}^{\ast^\top}\|_2 \| \mathbf{\hat{B}}^t \|_2 \| \mathbf{xx}^\top\|_F\label{cscscs} \\
    &\leq \frac{2\eta}{rn} \|\mathbf{F}\|_F \; \| \mathbf{{W}}^\ast \|_2\label{normm}\\
    &\leq 2\eta\frac{\delta_k}{1-\delta_k} \bar{\sigma}_{\max, \ast}^2 \label{ff}
\end{align}
where \eqref{perp0} follows since $\mathbf{\hat{B}}^{t^\top} \mathbf{\hat{B}}_\perp^t = \mathbf{0}$, \eqref{i0} follows since $\mathbf{\hat{B}}^{t^\top} \mathbf{\hat{B}}^t = \mathbf{I}_k$, \eqref{cscs} and \eqref{cscscs} follows by the Cauchy-Schwarz inequality, \eqref{normm} follows by the orthonormality of $\mathbf{\hat{B}}^t$ and $\mathbf{\hat{B}}^\ast$ and \eqref{ff} follows by Lemma \ref{lem:bound_f} and the definition of $\bar{\sigma}_{\max, \ast}$.
Next, again for any $\mathbf{x} \in \mathbb{R}^k: \|\mathbf{x}\|_2=1$, 
\begin{align}
- \frac{2\eta}{rn} \langle \mathbf{Q}^t,  \mathbf{F}\mathbf{xx}^\top \mathbf{\hat{B}}^{t^\top} \rangle &=  - \frac{2\eta}{rn}\text{tr}( (\mathbf{\hat{B}}^t \mathbf{\hat{B}}^{t^\top} \mathbf{\hat{B}}^\ast  \mathbf{{W}}^{\ast^\top} - \mathbf{\hat{B}}^t \mathbf{F}^\top - \mathbf{\hat{B}}^\ast  \mathbf{{W}}^{\ast^\top})  \mathbf{F}\mathbf{xx}^\top \mathbf{\hat{B}}^{t^\top} )\nonumber \\
&= - \frac{2\eta}{rn}\text{tr}( (\mathbf{\hat{B}}^t \mathbf{\hat{B}}^{t^\top}  - \mathbf{I}_d) \mathbf{\hat{B}}^\ast  \mathbf{{W}}^{\ast^\top } \mathbf{F}\mathbf{xx}^\top \mathbf{\hat{B}}^{t^\top} ) + \frac{2\eta}{rn} \text{tr} (  \mathbf{Fx} \mathbf{x}^\top \mathbf{\hat{B}}^{t^\top} \mathbf{\hat{B}}^t \mathbf{F}^\top) \nonumber \\
&= - \frac{2\eta}{rn}\text{tr}(  \mathbf{\hat{B}}^\ast  \mathbf{{W}}^{\ast^\top}  \mathbf{F}\mathbf{xx}^\top \mathbf{\hat{B}}^{t^\top} \mathbf{B}_{\perp}^t) + \frac{2\eta}{rn} \mathbf{x}^\top \mathbf{F}^\top \mathbf{Fx} \nonumber \\
&= \frac{2\eta}{rn} \mathbf{x}^\top \mathbf{F}^\top \mathbf{Fx} \nonumber \\
&\leq \frac{2\eta}{rn} \|  \mathbf{F}\|_2^2 \nonumber \\
&\leq {2\eta} \frac{\delta_k^2 }{(1-\delta_k)^2}\bar{\sigma}^2_{\max, \ast}
\end{align}
Thus, we have the following bound on the second term of \eqref{2terms}:
\begin{align}
    \max_{\mathbf{x}:\|\mathbf{x}\|_2=1} \frac{2\eta}{rn} \langle \mathbf{Q}^t, \mathbf{W}^{t+1} \mathbf{xx}^\top \mathbf{\hat{B}}^{t^\top} \rangle &\leq 2\eta \bar{\sigma}^2_{\max, \ast}\left(\frac{\delta_{k}}{1-\delta_{k}} + \frac{\delta_{k}^2}{(1-\delta_{k})^2}\right) \leq 4\eta \frac{\delta_{k}}{(1-\delta_{k})^2} \bar{\sigma}^2_{\max, \ast} \label{use}
\end{align}
since $0\leq \delta_k \leq 1 \implies \delta_{k}^2 \leq \delta_k$. 
Therefore, using \eqref{psd0}, \eqref{2terms}, \eqref{ddd} and \eqref{use}, we have
\begin{align}
    \sigma^2_{\min} ( \mathbf{R}_{t+1}) &\geq 1 - {2\eta}{\delta'_{k}} - {4\eta}\frac{\delta_{k}}{(1-\delta_{k})^2} \bar{\sigma}^2_{\max, \ast} \geq 1 -  {4\eta}\frac{\bar{\delta}_{k}}{(1-\bar{\delta}_{k})^2} \bar{\sigma}^2_{\max, \ast}  \label{91}
\end{align}
where $\bar{\delta}_k = {\delta}'_k + {\delta}_k$.
This means that \begin{align} \label{valid}
    \|(\mathbf{R}^{t+1})^{-1}\|_2 \leq \left(1- {4\eta} \frac{\bar{\delta}_k}{(1-\bar{\delta}_k)^2}\bar{\sigma}_{\max,\ast}^2 \right)^{-1/2}
\end{align}
Note that $1- {4\eta} \frac{\bar{\delta}_k}{(1-\bar{\delta}_k)^2}\bar{\sigma}_{\max,\ast}^2$ is strictly positive as long as $\frac{\bar{\delta}_k}{(1-\bar{\delta}_k)^2} < 1$, which we will verify shortly, due to our earlier assumption that $\eta \leq 1/(4\bar{\sigma}_{\max,\ast}^2 )$.
Therefore, from \eqref{justR}, we have
\begin{align}
    \text{dist}(\mathbf{\hat{B}}^{t+1}, \mathbf{\hat{B}}^\ast)
    &\leq \frac{1}{\sqrt{1- {4\eta} \frac{\bar{\delta}_k}{(1-\bar{\delta}_k)^2}\bar{\sigma}_{\max,\ast}^2 } } \left(1 - {\eta} \bar{\sigma}_{\min,\ast}^2 E_0 + {2\eta} \frac{\bar{\delta}_{k}}{(1-\bar{\delta}_{k})^2}\bar{\sigma}^2_{\max,\ast} \right)\; \text{dist}(\mathbf{\hat{B}}^t,\mathbf{\hat{B}}^\ast) \nonumber
\end{align}
Next, let $\bar{\delta}_k < 16E_0/(25 \cdot 5\kappa^2)$. This implies that $\bar{\delta}_k < 1/5$.  Then $\bar{\delta}_k/(1-\bar{\delta}_k)^2 < 25 \bar{\delta}_k/16 \leq E_0/(5\kappa^2) \leq 1$, validating \eqref{valid}. Further, it is easily seen that 
\begin{align}
    1- \eta E_0 \bar{\sigma}_{\min,\ast}^2 + \eta \frac{\bar{\delta}_k}{(1-\bar{\delta}_k)^2} \bar{\sigma}_{\max,\ast}^2 &\leq 1 - 4 \eta \frac{\bar{\delta}_k}{(1-\bar{\delta}_k)^2}  \bar{\sigma}_{\max,\ast}^2  \nonumber \\
    &\leq 1 -  \eta E_0\bar{\sigma}_{\min,\ast}^2 /2
\end{align}
Thus 
\begin{align}
    \text{dist}(\mathbf{\hat{B}}^{t+1}, \mathbf{\hat{B}}^\ast)
    &\leq \left(1 -  \eta E_0\bar{\sigma}_{\min,\ast}^2 /2\right)^{1/2}\; \text{dist}(\mathbf{\hat{B}}^t,\mathbf{\hat{B}}^\ast). \nonumber
\end{align}
Finally, recall that $\bar{\delta}_k = \delta_k + \delta_k' = c\left(\frac{k^{3/2}\sqrt{\log(rn)}}{\sqrt{m}} + \frac{k \sqrt{d}}{\sqrt{rnm}}\right)$ for some absolute constant $c$. Choosing $m \geq c'(\kappa^4 k^3\log(rn)/E_0^2 + \kappa^4 k^2 d/(E_0^2 rn))$ for another absolute constant $c'$ satisfies $\bar{\delta}_k \leq  16E_0/(25 \cdot 5\kappa^2)$. Also, we have conditioned on two events, described in Lemmas \ref{lem:bound_f} and \ref{lemma:sth}, which occur with probability at least $1- e^{-110d} - e^{-110k^2 \log(rn)} \geq 1 - e^{-100 \min (k^2\log(rn), d)}$, completing the proof.




\end{proof}

Finally, Theorem \ref{thrm:linear} follows by recursively applying Lemma \ref{lem:main} and taking a union bound over all $t\in [T]$.

\subsection{Initialization}


As mentioned in the main body, our interpretation of Theorem \ref{thrm:linear} assumes that the initial distance is bounded above by a constant less than one, i.e., $E_0$ is bounded below by a constant greater than zero. We can achieve such an initialization without increasing the overall sample complexity via the Method-of-Moments algorithm, ignoring log factors.  To show this, we adapt a result from \cite{tripuraneni2020provable}.

\begin{customthm}{2}
[Theorem 3, \cite{tripuraneni2020provable}] In addition to Assumptions \ref{assump:tasks} and \ref{assump:norm}, suppose that $\mathbf{x}_{i}^{0,j} \sim \mathcal{N}(0, \mathbf{I}_d)$ independently for all $i \in [n], j\in [m]$. 
If each client $i \in [n]$ sends the server $\mathbf{Z}_i\! \coloneqq\!\frac{1}{m}\sum_{j=1}^m (y_i^{0,j})^2 \mathbf{x}_i^{0,j} (\mathbf{x}_i^{0,j})^\top$ and the server computes $\mathbf{\hat{U}} \mathbf{D}\mathbf{\hat{U}}^\top\!\leftarrow\! \text{rank-}k \text{ SVD}(\tfrac{1}{n}\textstyle{\sum_{i=1}^n\mathbf{Z}_i)}$ and sets $\mathbf{B}^0 =\mathbf{\hat{U}}$. Then, if $m \geq c \polylog(d,mn) \tilde{\kappa}^2 kd/(\sigma_{\min,\ast}^2 n)$,
\begin{align}
    \dist\left(\mathbf{B}^0, \mathbf{{B}}^\ast\right) &\leq \tilde{O}\left(\frac{\tilde{\kappa}^2 kd}{\sigma_{\min,\ast}^2 mn} \right)
\end{align}
with probability at least $1 - O((mn)^{-100})$ for some absolute constant $c$, where $\sigma_{\min,\ast} \coloneqq \sigma_{\min}\left(\frac{1}{\sqrt{kn}}\mathbf{W}^\ast\right)$, $\tilde{\kappa}\coloneqq \frac{\sigma_{\max}\left(\frac{1}{\sqrt{kn}}\mathbf{W}^\ast\right)}{\sigma_{\min,\ast}}$ and $\tilde{O}(\cdot)$ hides log factors.
\end{customthm}

The above result is a direct adaptation of Theorem 3 in \cite{tripuraneni2020provable} so we omit the proof. Note that the $\tfrac{1}{\sqrt{k}}$ factor in the definition of $\sigma_{\min,\ast}$ is a  scaling factor to enforce consistency with the assumption that $\|\mathbf{w}^\ast_i\| = \Theta(1)$ in \cite{tripuraneni2020provable} (since we have assumed $\|\mathbf{w}^\ast_i\| = \sqrt{k}$).  This result shows that $m_{\text{init}} = \tilde{\Omega}(\frac{\tilde{\kappa}^2 kd}{\sigma_{\min,\ast}^2 n})$ samples are required for proper initialization.
Since $ \frac{1}{\sigma_{\min,\ast}^2} \leq k{\kappa^2}$ (as $\sigma_{\max,\ast}^2 \geq 1/k$, see \eqref{justR}), the overall sample complexity does not increase by more than log factors.

\subsection{Proof Challenges} \label{app:challenges}

We next discuss two analytical challenges involved in proving Theorem \ref{thrm:linear}.


{\em (i) Row-wise sparse measurements.} Recall that the measurement matrices $\mathbf{A}_{i,j}^t$ have non-zero elements only in the $i$-th row. This property is beneficial in the sense that it allows for distributing  computation across the $n$ clients. However, it also means that the operators $\{\frac{1}{\sqrt{m}}\mathcal{A}^t\}_t$ \emph{do not} satisfy Restricted Isometry Property (RIP), which therefore prevents us from using standard RIP-based analysis. The RIP is defined as follows:
\begin{definition}[Restricted Isometry Property]
An operator $\mathcal{B}: \mathbb{R}^{n \times d} \rightarrow \mathbb{R}^{nm}$ satisfies the $k$-RIP with parameter $\delta_k \in [0,1)$ if and only if
\begin{equation}
    (1-\delta_k) \|\mathbf{M}\|_{F}^2 \leq \|\mathcal{B}(\mathbf{M})\|_2^2 \leq (1+\delta_k) \|\mathbf{M}\|_{F}^2
\end{equation}
holds simultaneously for all $\mathbf{M} \in \mathbb{R}^{n\times d}$ of rank at most $k$.
\end{definition}

\begin{claim} \label{claim:no_rip}
Let ${\mathcal{A}}: \mathbb{R}^{r n \times d} \rightarrow \mathbb{R}^{rnm}$ such that ${\mathcal{A}}(\mathbf{M}) = [\langle \mathbf{e}_{i} (\mathbf{x}_i^j)^\top, \mathbf{M}\rangle]_{1\leq i \leq rn, 1\leq j \leq m}$, and let the samples $\mathbf{x}_i^j$ be i.i.d. sub-gaussian random vectors with mean $\mathbf{0}_d$ and covariance $\mathbf{I}_d$.
Then if $m \leq d/2$, with probability at least $1 - e^{-cd}$ for some absolute constant $c$, $\frac{1}{\sqrt{m}}{\mathcal{A}}$ does not satisfy 1-RIP for any constant $\delta_1 \in [0,1)$.
\end{claim}
\begin{proof}
Let $\mathbf{M} = \mathbf{e}_1 (\mathbf{x}_1^1)^\top$.
Then \begin{align}
    \|\frac{1}{\sqrt{m}}\mathcal{A}(\mathbf{M})\|_2^2 &= \frac{1}{{m}}\sum_{i=1}^{rn} \sum_{j=1}^m \langle \mathbf{e}_i (\mathbf{x}_i^j)^\top , \mathbf{e}_1 (\mathbf{x}_1^1)^\top\rangle^2 \nonumber \\
    &= \frac{1}{{m}} \|\mathbf{x}_1^1 \|_2^4 + \frac{1}{m} \sum_{j=2}^m \langle \mathbf{x}_1^j, \mathbf{x}_1^1 \rangle^2  \nonumber \\
    &\geq \frac{1}{{m}} \|\mathbf{x}_1^1 \|_2^4 
\end{align}
Also observe that $\|\mathbf{M}\|_F^2 = \|\mathbf{x}_1^1 \|_2^2$. 
Therefore, we have
\begin{align}
 \mathbb{P}\left(\frac{ \left\|\frac{1}{\sqrt{m}}\mathcal{A}(\mathbf{M})\right\|_2^2}{\left\|\mathbf{M}\right\|_F^2} \geq \frac{d}{2m}\right) &\geq  \mathbb{P}\left(\frac{ \frac{1}{{m}} \left\|\mathbf{x}_1^1 \right\|_2^4 }{\left\|\mathbf{x}_1^1 \right\|_2^2} \geq \frac{d}{2m}\right) \nonumber \\ 
&= \mathbb{P}\left( \left\|\mathbf{x}_1^1 \right\|_2^2 \geq \frac{d}{2}\right)  \nonumber \\
    &=  1 - \mathbb{P}\left(\|\mathbf{x}_1^1 \|_2^2 - d \leq \frac{-d}{2}\right) \nonumber \\
    &\geq 1 - e^{-c d}
\end{align}
where the last inequality follows for some absolute constant $c$ by the sub-exponential property of $\|\mathbf{x}_1^1 \|_2^2$ and the fact that $\mathbb{E}[\|\mathbf{x}_1^1 \|_2^2] = d$.
Thus, with probability at least $1 - e^{-c d}$, $\left\|\frac{1}{\sqrt{m}}\mathcal{A}(\mathbf{M})\right\|_2^2 \geq \frac{d}{2m} \left\|\mathbf{M}\right\|_2^2$, meaning that $\frac{1}{\sqrt{m}}\mathcal{A}$ does not satisfy 1-RIP with high probability if $m \leq \frac{d}{2}$.
\end{proof}

Claim \ref{claim:no_rip} shows that we cannot use the RIP to show $\mathcal{O}(d/(rn))$ sample complexity for $m$ - instead, this approach would require $m = \Omega(d)$. Fortunately, we do not need concentration of the measurements for all rank-$k$ matrices $\mathbf{M}$, but only a particular class of rank-$k$ matrices that are {\em row-wise incoherent}, due to the row-wise incoherence of $\mathbf{W}^\ast$ (see Assumption \ref{assump:norm} and Definition \ref{def:inco}). Leveraging the row-wise incoherence of the matrices being measured allows us to show that we only require $m = \Omega(k^3 \log(rn) + k^2 d/(rn))$ samples per user (ignoring dimension-independent constants).

{\em (ii) Non-symmetric updates.} 
Existing analyses for nonconvex matrix sensing study algorithms with symmetric update schemes for the factors $\mathbf{W}$ and $\mathbf{B}$, either alternating minimization, e.g. \citep{Jain_2013}, or alternating gradient descent, e.g. \citep{tu2016low}. 
Here we show contraction due to the gradient descent step in principal angle distance, differing from the standard result for gradient descent using Procrustes distance  \citep{tu2016low, zheng2016convergence, park2018finding}. 
We combine aspects of both types of analysis in our proof. 

\bibliography{refs}
\bibliographystyle{plainnat}



\end{document}